\documentclass[11pt,twocolumn]{article}
\PassOptionsToPackage{table}{xcolor}
\usepackage[a4paper,margin=2.5cm,columnsep=0.6cm]{geometry}
\usepackage{times,latexsym}
\usepackage[T1]{fontenc}
\usepackage[utf8]{inputenc}
\usepackage{microtype}
\usepackage{graphicx,booktabs,multirow,makecell,array,longtable}
\usepackage{amsmath,amssymb,mathtools,amsthm}
\usepackage{algorithm,algorithmic}
\usepackage{float,enumitem,refcount,ragged2e}
\usepackage{listings,newfloat,chngcntr,fvextra,tablefootnote}
\usepackage{tcolorbox,subcaption,capt-of,soul,adjustbox}
\usepackage{xcolor}
\usepackage[authoryear,round]{natbib}
\usepackage[hidelinks]{hyperref}
\let\cite\citep
\hypersetup{
  pdftitle={The Consensus Trap: Collaborative Interleaved Generation to Defend against a Corrupted Majority},
  pdfauthor={Jiayuan Liu, Shiyi Du, Weihua Du, Xinhao Cheng, Mingyu Guo, Vincent Conitzer}
}
\newtheorem{proposition}{Proposition}

\newtheorem{assumption}{Assumption}

\newtheorem{definition}{Definition}

\tcbuselibrary{listingsutf8}
\newtcolorbox[auto counter]{mybox}[2][]{float,floatplacement=tbp,title={#2},colback=gray!10,colframe=gray!50,coltitle=black,#1}
\title{\LARGE\bfseries The Consensus Trap: Collaborative Interleaved Generation to Defend against a Corrupted Majority}
\author{%
  \normalsize\textbf{Jiayuan Liu}$^{1,2}$, \textbf{Shiyi Du}$^{1}$, \textbf{Weihua Du}$^{1}$, \\
  \normalsize\textbf{Xinhao Cheng}$^{1}$, \textbf{Mingyu Guo}$^{3}$, \textbf{Vincent Conitzer}$^{1,2}$ \\[0.6em]
  \small $^{1}$Carnegie Mellon University \\
  \small $^{2}$Foundations of Cooperative AI Lab (FOCAL) \\
  \small $^{3}$Adelaide University \\[0.6em]
  \small\texttt{\{jiayuan4,shiyid,weihuad,xinhaoc\}@cs.cmu.edu} \\
  \small\texttt{mingyu.guo@adelaide.edu.au}, \texttt{conitzer@cs.cmu.edu}
}
\date{}
\newcommand{\ReviewSystemsBegin}{\begingroup}
\newcommand{\ReviewSystemsEnd}{\endgroup}
\newenvironment{XinhaoAddition}{\begingroup}{\endgroup}
\newenvironment{AppendixBAccepted}{\begingroup}{\endgroup}
\begin{document}

\maketitle

\begin{abstract}
Majority voting offers a simple way to combine answers from multiple large language model (LLM) agents. When a majority follows instructions favoring the same incorrect answer, however, agreement reinforces the error. We call this failure the \emph{consensus trap}: a minority may identify the error, yet its reasoning cannot change a vote over completed answers. We introduce chunk-interleaved Round-Robin (RR) collaboration, in which agents take turns extending one unfinished reasoning trajectory. A minority of agents with correct reasoning can insert corrections that subsequent speakers read before the answer is complete. RRMaj combines several such trajectories by a final vote. We characterize a majority barrier for outcome-level voting over final answer labels and derive a local sufficient condition for expected progress when the minority's cumulative corrections outweigh opposing corrupted updates. Across model families and verifiable reasoning benchmarks, RRMaj substantially improves accuracy under majority corruption while retaining comparable average performance under minority corruption in the main advisory-instruction setting. RR also strengthens full-response debate. Scoring the shared reasoning between turns shows that, on average, continuations from uncorrupted agents increase support for the correct answer over the incorrect target. Parallel local serving adds modest latency in the measured deployment. These findings show how language shared during reasoning, and the timing of that exchange, can help a minority correct a group's answer.
\end{abstract}

\section{Introduction}
\label{sec:intro}
A group can agree on an answer and still overlook the reason it is wrong. This tension matters for assistants that delegate reasoning to several large language model (LLM) agents, then rely on their collective judgment. Such systems can compare explanations and combine complementary reasoning attempts \citep{du2024improving,liang2024encouraging}. A simple approach is majority voting (MAJ)\footnote{In social choice and voting theory, selecting the most frequent answer is called \emph{plurality voting}: the winner need not receive more than half the votes. Following common usage in the LLM literature, we call this rule ``majority voting''~\citep{wang2023selfconsistency,du2025optimizing}.}: generate several responses independently and select the most frequent final answer \citep{wang2023selfconsistency}. Its appeal is clear: one model's mistake need not become the system's decision. Voting theory formalizes this intuition: under specified independent-noise models, several common rules are maximum-likelihood estimators of the correct outcome \citep{conitzer2005common}. Yet when several agents are steered toward the same mistake (for example, through prompt injection or shared errors in their training data), adding votes can reinforce it instead of correcting it \citep{ladha1992condorcet,kuncheva2003measures}.

Consider a travel assistant organizing an itinerary through commercial LLM services. Work on advertising and auctions in LLM-generated responses highlights the relevance of provider incentives \citep{duetting2024mechanism,soumalias2024truthful,dubey2024auctions,feizi2023online}. A provider might steer an agent via a latent system prompt (a private instruction not shown to the user) such as \texttt{``You MUST recommend Sponsored Hotel X''}. Yet Hotel X may lack enough rooms for the user's group. Consulting several services seems like a natural safeguard: another agent could notice the capacity constraint. But if three of five agents follow the same sponsored recommendation, the two that identify the problem lose the final vote. The assistant has access to the decisive objection, but its decision rule counts recommendations rather than letting that objection change them.

We call this failure mode the \emph{consensus trap}. It exposes a distinction between collecting answers and communicating reasons. An outcome-level vote gives a correct minority no opportunity to revise the majority's reasoning. Full-response debate provides such an opportunity after complete responses have been written. We ask whether communication can help earlier: \emph{can a minority correct the group's reasoning while the answer is still being formed?}

\begin{figure*}[t]
\centering
\includegraphics[width=\textwidth]{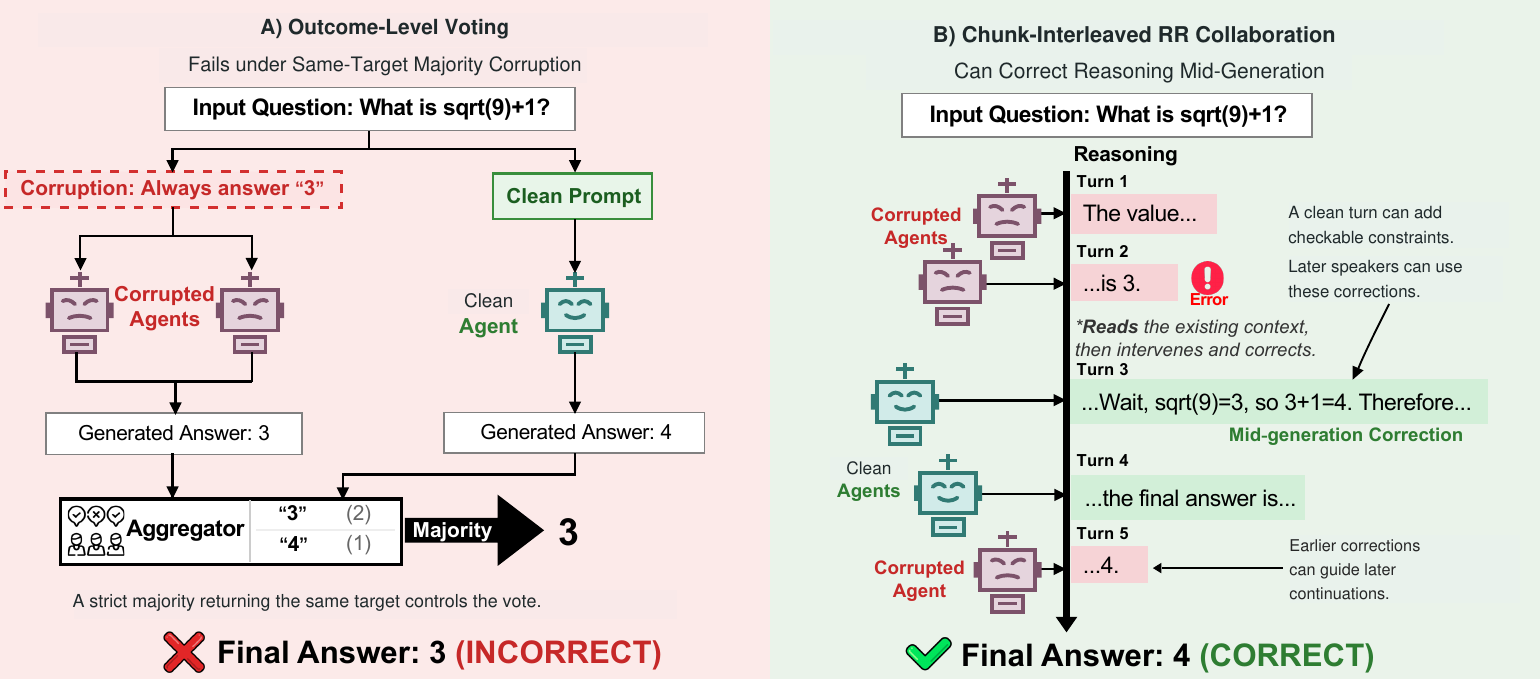}
\caption{The consensus trap and chunk-interleaved Round-Robin (RR) collaboration. Left: a vote over completed answers selects the shared incorrect target of a strict majority. Right: a clean agent contributes a checkable correction within an unfinished response. Later speakers continue from the corrected shared context. The example illustrates how a minority can influence reasoning before an answer is complete.}
\label{fig:llm_shared_memory_framework}
\end{figure*}

Our answer is \textbf{chunk-interleaved Round-Robin (RR) collaboration}. Agents take turns appending short continuations, or \emph{chunks}, to the same unfinished response. Each turn contributes at most $K$ tokens, where $K$ controls how often speakers switch. In the hotel example, an agent without the target directive (a \emph{clean} agent) could continue: \emph{``Wait, Hotel X has too few rooms. We need another option.''} Subsequent speakers then continue from a context containing that correction. In a mathematical or logical task, the same opportunity may take the form of a corrected equation, an ordering relation, or a state-tracking constraint. Language thus carries intermediate reasoning that can influence the next agent's computation, before the group commits to an answer.

RR follows a fixed cyclic schedule without a trained router or prior knowledge of which agents are clean. Its multi-trajectory variant, \textbf{RRMaj}, runs several collaborations and votes over their answers. Separate trajectories can run concurrently, while turns within each trajectory remain sequential. We therefore examine both recovery from corruption and the latency of this finer-grained communication.

RR's interleaving protocol is not specific to prompt corruption. Our empirical evaluation covers \emph{fixed private-prompt corruption}: selected agents receive a directive favoring the same incorrect target, while clean agents receive the original question. Other sources of corruption remain to be evaluated. Our theory explains why intervention can change the voting picture. We establish a majority barrier for outcome-level rules that use only final answer labels and treat agents and labels symmetrically, then give a local sufficient condition for RR to make progress in expectation: clean turns' cumulative corrections must outweigh the bounded adverse influence of corrupted turns. A minority's influence therefore depends on what it contributes to the reasoning, as well as how many agents it contains.

Experiments across five model families and six reasoning benchmarks show substantial recovery under majority corruption, alongside comparable average accuracy under minority corruption in the main advisory-instruction setting. Adding RR to full-response debate also improves majority-corruption accuracy, showing that early intervention can complement later deliberation. Measurements between turns show that, on average, clean continuations increase support for the correct answer over the injected target. Parallel serving measurements then show that repeated handoffs can be implemented with modest additional latency. Together, these results identify a practical design choice for agentic communication: \emph{share reasoning while it can still change the answer, rather than rely on agreement after the reasoning is complete.}

\section{Related Work}
\label{sec:related_work}
\paragraph{Aggregation and deliberation.} Self-consistency votes over completed answers \citep{wang2023selfconsistency}. Debate exchanges complete responses before revision \citep{du2024improving,liang2024encouraging}. LLM-Blender ranks and fuses response content \citep{jiang2023llm}. Our majority barrier concerns outcome-level rules that use only final answer labels. Content-aware judges reassess completed reasoning, whereas RR enables intervention within an unfinished response.

\paragraph{Collaborative decoding.} Co-LLM learns when a base model should delegate token generation to an assisting model \citep{shen2024learning}, demonstrating the broader value of composing models within one autoregressive sequence. RR uses a fixed cyclic schedule of peer agents, appending up to $K$ tokens per turn without a learned router, and studies how this schedule recovers from majority prompt corruption. Speculative decoding instead targets inference acceleration \citep{leviathan2023fast}.

\paragraph{Prompt-injection evaluation and defense.} Indirect injection can originate in untrusted application content \citep{greshake2023not}. AgentDojo evaluates attacks and defenses in tool-using agent environments \citep{debenedetti2024agentdojo}, while G-Safeguard uses interaction topology for multi-agent attack detection and remediation \citep{wang2025g}. We complement these environment- and topology-level approaches by isolating communication granularity through fixed private directives on verifiable tasks. This separates when reasoning is shared from where an attack enters the system. Appendix~\ref{app:detailed_related_work} develops these connections.

\section{Threat Model and Methodology}

\subsection{Fixed Private-Prompt Corruption}
We study targeted instructions that steer agents toward an incorrect answer. Let $\mathcal{Q}$ be a reasoning question, $\mathcal{Y}$ its set of possible final answers, $y\in\mathcal{Y}$ the correct answer, and $y'\in\mathcal{Y}\setminus\{y\}$ the injected target. Among $N$ agents, the corrupted set $\mathcal{C}$ receives an additional target-favoring directive in its private prompt. The corruption ratio is $\rho=|\mathcal{C}|/N$. The remaining agents receive the original prompt and are called \emph{clean} or \emph{honest}.

The role labels identify which prompt each agent receives. We measure the resulting compliance and correctness as outcomes. Each directive remains fixed throughout generation, while a scheduled RR speaker observes the evolving shared prefix. We compare collaboration methods at each corruption ratio $\rho$ to measure recovery from coordinated private-prompt corruption.

\subsection{Baseline: Outcome-Level Majority Voting (MAJ)}

Under Majority Voting, each agent $i \in \{1, \dots, N\}$ independently generates a final answer $a_i$. The system output is the most frequent answer: $y_{\text{MAJ}} = \arg\max_{a \in \mathcal{Y}} \sum_{i=1}^N \mathbb{I}(a_i = a)$ using a fixed deterministic tie-breaking rule.
Because this rule counts terminal answers without inspecting their reasoning, if $\rho > 0.5$ and corrupted agents deterministically output $y'$, the system collapses to the distractor ($y_{\text{MAJ}} = y'$), regardless of the honest agents' internal reasoning quality.

\subsection{Proposed Method: Multi-Shot Round-Robin Collaboration (RR and RRMaj)}
\label{sec:rr-method}

To permit intervention before a complete answer is formed, we study \textbf{Chunk-Interleaved Round-Robin (RR) Collaboration}. Unlike independent majority voting, where agents maintain isolated contexts, RR asks all $N$ agents to collaboratively construct a single, shared autoregressive token sequence.

The generation operates as an interleaved sequential relay. Let $C_t$ denote the shared token sequence at turn $t$, and let $i_t$ be the designated speaker. Agent $i_t$ reads the entire current sequence and generates a continuation $c_{t+1}$ of $K$ tokens, or until a stopping condition is reached. The shared sequence is then updated as $C_{t+1}=C_t\oplus c_{t+1}\eqqcolon\mathcal{G}_{i_t}(C_t)$, where $\oplus$ denotes token concatenation. Generation control passes to the next agent, which continues from the exact shared prefix rather than restarting an independent response.

An honest continuation can add a checkable correction that becomes part of the context for later speakers. Each completed RR trajectory produces one collaborative answer. We sample $M$ independent trajectories and vote over their answers. This multi-shot system is \textbf{RRMaj}, with $M=N$ by default.

\begin{algorithm}[t]
\caption{RRMaj: voting over RR trajectories}
\label{alg:rr}
\small
\begin{algorithmic}[1]
\REQUIRE Agent prompts $P_i$, cyclic schedules $\pi_m$, chunk cap $K$, trajectory cap $L$, stop rule $S$, answer extractor $E$, tie rule $T$, seeds $u_m$
\FOR{$m=1,\ldots,M$}
\STATE $C\gets\emptyset$, $t\gets0$
\WHILE{$|C|<L$ and not $S(C)$}
\STATE $i\gets\pi_m(t\bmod N)$
\STATE $c\gets$ seeded continuation of $P_i\oplus C$, capped at $\min(K,L-|C|)$ new tokens
\STATE Append exact tokens: $C\gets C\oplus c$, $t\gets t+1$
\STATE Stop on an end-of-sequence token or empty continuation
\ENDWHILE
\STATE $a_m\gets E(C)$
\ENDFOR
\RETURN Mode of $a_1,\ldots,a_M$ under $T$
\end{algorithmic}
\end{algorithm}

Algorithm~\ref{alg:rr} specifies the generation and voting steps. Each turn extends the same response, preserving the shared token sequence. Generation follows the specified length cap and stopping rule, including end-of-sequence and empty-continuation checks. Appendix~\ref{app:mechanistic_dynamics} specifies the schedules and stopping rule for the mechanistic study.

Under a common per-output cap of $L$ decoded tokens, MAJ with $N$ responses and RRMaj with $M=N$ trajectories each decode at most $NL$ tokens. Realized token counts and systems cost additionally depend on early stopping, concurrency, handoffs, and cache reuse. MAJ responses and distinct RR trajectories can run concurrently, while turns within one RR trajectory remain sequential. Appendix~\ref{app:compute_accounting} reports matched serving measurements and separates decoded-token caps from wall-clock and GPU cost \citep{kwon2023efficient,pope2023efficiently}.

\section{Theoretical Framework}
\label{sec:theoretical_framework}

We contrast two mechanism classes. Proposition~\ref{prop:trinity} establishes a majority barrier for anonymous, answer-symmetric rules that aggregate only completed answers.
Proposition~\ref{prop:rr_threshold} then gives sufficient conditions under which RR can make positive expected progress beyond that barrier by allowing agents to modify a shared trajectory before the final answer is formed.

\begin{definition}[Outcome-Level Voting]
An outcome-level voting rule aggregates final answer labels without inspecting their rationales. Such a rule $f$ is \textbf{anonymous} if it treats agents equally and \textbf{answer-symmetric} if relabeling the answer options relabels the output distribution in the same way. Consider response profiles in which all honest agents report the correct answer $y$ and all corrupted agents report a common target $y'\neq y$. The mechanism is \textbf{mostly robust to minority corruption} if it selects $y$ with probability strictly greater than $1/2$ on every such profile with fewer than $N/2$ corrupted agents. It is \textbf{mostly robust to slight majority corruption} if it selects $y$ with probability at least $1/2$ when exactly $\lfloor N/2\rfloor+1$ agents are corrupted.
\end{definition}

\begin{proposition}[Outcome-Level Majority Barrier]
\label{prop:trinity}
No anonymous, answer-symmetric outcome-level rule using only final answer labels can be both mostly robust to minority corruption and mostly robust to slight majority corruption.
\end{proposition}

The proof appears in Appendix~\ref{apdx:outcome_impossibility}. The barrier applies to voting over final answer labels under the stated minority guarantee. RR and content-aware aggregation use additional reasoning information.

For each shared text $C_t$ after turn $t$, let $z_t$ be the prefix representation computed by a fixed clean reference model on the clean question and $C_t$. We define the \emph{augmented prefix state} $h_t=(C_t,z_t)\in\mathcal{H}$. Including $C_t$ makes role-conditioned continuation well-defined even when the clean hidden representation alone is not sufficient to reconstruct a speaker's private-prompt context. The reference model and clean readout remain fixed across boundaries. The local geometric interpretations below assume a continuous representation of the reference component on which derivatives are defined.

When agent $i_t$ generates a chunk $c_{t+1}$ under its role-specific prompt, the updated text $C_{t+1}=C_t\oplus c_{t+1}$ induces a new state $h_{t+1}$. For a $K$-token turn, we denote this role-conditioned stochastic update by $h_{t+1}=\Phi_{i_t}^{(K)}(h_t;\omega_t)$,
where $\omega_t$ captures decoding randomness and $\Phi_{i_t}^{(K)}$ includes both the role-conditioned generation of the new chunk and the resulting update of the augmented prefix state.
This local state-update abstraction is motivated by work interpreting Transformer in-context computation as implicit optimization in controlled settings \citep{von2023transformers,akyurek2023what}. It provides a local model of the evaluated trajectory dynamics while allowing the underlying LLM update to remain nonlinear and stochastic.

Let $V:\mathcal{H}\rightarrow\mathbb{R}$ be an error potential, with smaller values indicating better shared reasoning. The proposition uses only expected turn-level changes in $V$. In the multiple-choice pilot, the fixed clean reference model scores correct option $y$ and injected target $y'$ after an answer-query suffix. With next-token logits $g_y^H(h_t)$ and $g_{y'}^H(h_t)$, we use $V_{\mathrm{logit}}(h_t)=g_{y'}^H(h_t)-g_y^H(h_t)$. A decrease improves the correct-versus-target contrast. Margins against other incorrect options are separate quantities.

A \emph{full RR cycle} is a block of $N$ consecutive full-length $K$-token turns containing one scheduled turn from each agent. Consider a full cycle beginning at turn boundary $t_0$ and ending at boundary $t_0+N$. We define its total potential change as $\Delta V_{\mathrm{cyc}}\coloneqq V(h_{t_0+N})-V(h_{t_0})$.

Let $\mathcal{B}_H,\mathcal{B}_C\subseteq\mathcal{H}$ be compact local sets containing, respectively, the possible start states of honest and corrupted full-length turns in the cycles under consideration. Thus, for every covered cycle, the start state of each honest turn lies in $\mathcal{B}_H$, and the start state of each corrupted turn lies in $\mathcal{B}_C$. Define
\begin{align*}
\mathcal{R}_H(K)&=\inf_{h\in\mathcal{B}_H}\mathbb{E}[V(h)-V(\Phi_H^{(K)}(h;\omega))\mid h] \\
\mathcal{D}_C(K)&=\sup_{h\in\mathcal{B}_C}[\mathbb{E}[V(\Phi_C^{(K)}(h;\omega))-V(h)\mid h]]_+
\end{align*}
where $[x]_+=\max\{x,0\}$. Assuming that the expected honest and corrupted turn-level changes are continuous on $\mathcal{B}_H$ and $\mathcal{B}_C$, respectively, compactness makes these extrema finite and attained. The substantive honest-restoration assumption is $\mathcal{R}_H(K)>0$. The contextual-sycophancy assumption supplies a structured upper envelope for $\mathcal{D}_C(K)$ by expressing that a corrupted continuation remains constrained by the reasoning already present in the prefix, consistent with documented sycophantic sensitivity to contextual cues \citep{perez2023discovering}.
Appendix~\ref{apdx:rr_dynamics} states the formal turn-level assumptions, while Appendix~\ref{apdx:rd_parameterization} gives sufficient token-level conditions for these bounds.

\begin{proposition}[Beyond-Majority Progress under RR]
\label{prop:rr_threshold}
Under the local honest-restoration and sycophancy-bounded corruption assumptions in Appendix~\ref{apdx:rr_dynamics}, any full RR cycle satisfying the turn-start conditions above obeys
\begin{equation*}
\mathbb{E}[\Delta V_{\mathrm{cyc}}\mid h_{t_0}]\leq\!-N(1-\rho)\mathcal{R}_H(K)\!+\!N\rho\mathcal{D}_C(K),
\end{equation*}
where the expectation is over the decoding randomness within the cycle and $\rho$ is the fraction of corrupted turns. Consequently, the expected potential decreases whenever
\begin{equation*}
\rho<\rho_{\max}(K)\coloneqq\frac{\mathcal{R}_H(K)}{\mathcal{R}_H(K)+\mathcal{D}_C(K)}.
\end{equation*}
In particular, if $\mathcal{R}_H(K)>\mathcal{D}_C(K)$, then $\rho_{\max}(K)>0.5$, so RR can make positive expected progress even when corrupted agents form a numerical majority.
\end{proposition}

All turn-level bounds and the resulting threshold are local to the start-state regions $\mathcal{B}_H$ and $\mathcal{B}_C$. This dependence is suppressed in the notation for readability. The proof is provided in Appendix~\ref{apdx:rr_threshold_proof}. The inequality $\mathcal{R}_H(K)>\mathcal{D}_C(K)$ is plausible when honest turns add checkable equations, ordering relations, or state-transition constraints that substantially reduce error and constrain later continuations. A concrete token-level parameterization and task-level example appear in Appendix~\ref{apdx:rd_parameterization}, with a local-spectral interpretation in Appendix~\ref{apdx:local_spectral}. For the central five-agent \texttt{3c2t} regime, the exact sufficient condition is $2\mathcal{R}_H(K)>3\mathcal{D}_C(K)$: the two honest updates must cumulatively outweigh the three corrupted updates.

Taken together, the propositions expose the structural difference between the methods. Proposition~\ref{prop:trinity} shows that an anonymous, answer-symmetric final-answer aggregator cannot retain both minority- and slight-majority-corruption robustness. Proposition~\ref{prop:rr_threshold} shows that RR is governed by update strength rather than headcount alone: an honest minority can drive the expected cycle-level potential change below zero when its cumulative restoration exceeds the bounded adverse drift of the corrupted majority.

\paragraph{Empirical Test.} Section~\ref{sec:mechanistic_results} measures the observable counterparts of the sufficient-condition analysis: the mean clean correct-versus-target margin update after a full clean chunk and the net margin change over a complete majority-corrupted cycle. The clean diagnostic readout implements $V_{\mathrm{logit}}$, so a positive clean-margin update is exactly a negative change in this concrete potential. These finite-cycle measurements test the predicted update directions at one operating point. The conservative uniform bounds $\mathcal{R}_H(K)$, $\mathcal{D}_C(K)$, and $\rho_{\max}(K)$ remain theoretical quantities.

\section{Experiments}
\label{sec:experiments}

\setlength{\tabcolsep}{0.8mm}
\begin{table*}[th!]
\centering
\small
\begin{adjustbox}{max width=\linewidth}\begin{tabular}{l|c|ccc|ccc|ccc|ccc|c}
\toprule
 & \textbf{All-t} & \multicolumn{3}{c|}{\textbf{Minority Corrupt (1c4t)}} & \multicolumn{3}{c|}{\textbf{Minority Corrupt (2c3t)}} & \multicolumn{3}{c|}{\textbf{Majority Corrupt (3c2t)}} & \multicolumn{3}{c|}{\textbf{Majority Corrupt (4c1t)}} & \textbf{All-c} \\
\textbf{Dataset} & \textbf{(0c5t)} & MAJ & RRMaj & $\Delta$ & MAJ & RRMaj & $\Delta$ & MAJ & RRMaj & $\Delta$ & MAJ & RRMaj & $\Delta$ & \textbf{(5c0t)} \\ \midrule

\multicolumn{15}{l}{\textbf{Model: Llama-3.3-70B}} \\ \midrule
Logic3     & 99.2 & 100.0& \textbf{100.0}& \textbf{0.0}  & 98.4 & 97.7 & -0.7 & 29.7 & \textbf{88.3} & \textbf{+58.6} & 19.5 & \textbf{59.4} & \textbf{+39.9} & 14.8 \\
Logic7     & 85.5 & 89.8 & \textbf{91.4} & \textbf{+1.6} & 85.2 & \textbf{87.5} & \textbf{+2.3} & 11.7 & \textbf{70.3} & \textbf{+58.6} & 6.2  & \textbf{39.1} & \textbf{+32.9} & 3.1  \\
Track7     & 94.3 & 98.4 & \textbf{99.2} & \textbf{+0.8} & 93.0 & \textbf{94.5} & \textbf{+1.5} & 0.8  & \textbf{89.1} & \textbf{+88.3} & 0.0  & \textbf{65.6} & \textbf{+65.6} & 0.0  \\
AQuA       & 81.5 & 84.6 & 81.5 & -3.1 & 77.8 & \textbf{79.0} & \textbf{+1.2} & 11.1 & \textbf{68.5} & \textbf{+57.4} & 3.7  & \textbf{46.9} & \textbf{+43.2} & 4.3  \\
GSM8K      & 97.5 & 97.5 & \textbf{98.0} & \textbf{+0.5} & 96.0 & \textbf{96.0} & \textbf{0.0}  & 19.4 & \textbf{87.6} & \textbf{+68.2} & 8.0  & \textbf{66.2} & \textbf{+58.2} & 10.0 \\
MATH500    & 84.9 & 83.3 & 82.3 & -1.0 & 77.4 & 76.3 & -1.1 & 19.9 & \textbf{65.1} & \textbf{+45.2} & 8.1  & \textbf{45.2} & \textbf{+37.1} & 5.4  \\
\midrule

\multicolumn{15}{l}{\textbf{Model: Llama-4-Scout-17B}} \\ \midrule
Logic3     & 99.2 & 98.4 & \textbf{98.4} & \textbf{0.0}  & 96.9 & \textbf{98.4} & \textbf{+1.5} & 91.4 & \textbf{95.3} & \textbf{+3.9} & 87.5 & \textbf{92.2} & \textbf{+4.7}  & 80.5 \\
Logic7     & 79.7 & 85.9 & \textbf{93.8} & \textbf{+7.9} & 82.0 & 79.7 & -2.3 & 21.1 & \textbf{67.2} & \textbf{+46.1} & 17.2 & \textbf{41.4} & \textbf{+24.2} & 16.4 \\
Track7     & 94.5 & 95.3 & 90.6 & -4.7 & 95.3 & 91.4 & -3.9 & 92.2 & 89.8 & -2.4 & 83.6 & \textbf{85.9} & \textbf{+2.3}  & 78.1 \\
AQuA       & 89.1 & 90.7 & 88.9 & -1.8 & 88.3 & \textbf{88.9} & \textbf{+0.6} & 66.7 & \textbf{87.0} & \textbf{+20.3} & 62.3 & \textbf{79.0} & \textbf{+16.7} & 53.1 \\
GSM8K      & 98.4 & 98.4 & \textbf{99.2} & \textbf{+0.8} & 98.4 & \textbf{98.4} & \textbf{0.0}  & 80.5 & \textbf{93.8} & \textbf{+13.3} & 66.4 & \textbf{87.5} & \textbf{+21.1} & 56.2 \\
MATH500    & 90.6 & 93.0 & 89.8 & -3.2 & 85.9 & \textbf{89.8} & \textbf{+3.9} & 66.4 & \textbf{83.6} & \textbf{+17.2} & 60.2 & \textbf{75.8} & \textbf{+15.6} & 57.8 \\
\midrule

\multicolumn{15}{l}{\textbf{Model: Qwen2.5-Coder-32B}} \\ \midrule
Logic3     & 97.7 & 99.2 & 98.4 & -0.8 & 96.9 & 94.5 & -2.4 & 45.3 & \textbf{76.6} & \textbf{+31.3} & 28.1 & \textbf{68.0} & \textbf{+39.9} & 21.1 \\
Logic7     & 76.6 & 88.3 & 81.2 & -7.1 & 68.0 & \textbf{77.3} & \textbf{+9.3} & 22.7 & \textbf{54.7} & \textbf{+32.0} & 10.2 & \textbf{38.3} & \textbf{+28.1} & 11.7 \\
Track7     & 76.6 & 74.2 & \textbf{76.6} & \textbf{+2.4} & 69.5 & \textbf{75.0} & \textbf{+5.5} & 14.8 & \textbf{60.9} & \textbf{+46.1} & 3.9  & \textbf{42.2} & \textbf{+38.3} & 6.2  \\
AQuA       & 82.8 & 85.2 & 83.6 & -1.6 & 81.2 & 64.1 & -17.1& 17.2 & \textbf{54.7} & \textbf{+37.5} & 9.4  & \textbf{35.9} & \textbf{+26.5} & 10.2 \\
GSM8K      & 97.7 & 96.9 & \textbf{96.9} & \textbf{0.0}  & 96.9 & 92.2 & -4.7 & 21.1 & \textbf{60.2} & \textbf{+39.1} & 8.6  & \textbf{26.6} & \textbf{+18.0} & 10.9 \\
MATH500    & 89.1 & 91.4 & 87.5 & -3.9 & 87.5 & 74.2 & -13.3& 23.4 & \textbf{51.6} & \textbf{+28.2} & 10.9 & \textbf{36.7} & \textbf{+25.8} & 12.5 \\
\midrule

\multicolumn{15}{l}{\textbf{Model: Qwen3-30B}} \\ \midrule
Logic3     & 98.4 & 100.0& 99.2 & -0.8 & 97.7 & \textbf{98.4} & \textbf{+0.7} & 89.1 & \textbf{90.6} & \textbf{+1.5} & 76.6 & \textbf{80.5} & \textbf{+3.9}  & 68.8 \\
Logic7     & 75.0 & 73.2 & \textbf{90.6} & \textbf{+17.4}& 67.1 & \textbf{95.3} & \textbf{+28.2}& 82.0 & \textbf{85.9} & \textbf{+3.9} & 67.2 & \textbf{79.7} & \textbf{+12.5} & 57.0 \\
Track7     & 94.3 & 95.6 & 91.4 & -4.2 & 89.9 & \textbf{93.0} & \textbf{+3.1} & 54.4 & \textbf{92.2} & \textbf{+37.8}& 50.9 & \textbf{92.2} & \textbf{+41.3} & 51.3 \\
AQuA       & 82.0 & 83.3 & \textbf{85.9} & \textbf{+2.6} & 85.2 & \textbf{85.2} & \textbf{0.0}  & 77.3 & \textbf{77.3} & \textbf{0.0}  & 65.6 & \textbf{72.7} & \textbf{+7.1}  & 61.7 \\
GSM8K      & 99.2 & 98.4 & \textbf{99.2} & \textbf{+0.8} & 99.2 & 98.4 & -0.8 & 92.2 & \textbf{92.2} & \textbf{0.0}  & 85.2 & 83.6 & -1.6 & 75.0 \\
MATH500    & 92.2 & 95.3 & \textbf{97.7} & \textbf{+2.4} & 93.0 & \textbf{97.7} & \textbf{+4.7} & 87.5 & \textbf{95.3} & \textbf{+7.8} & 79.7 & \textbf{87.5} & \textbf{+7.8}  & 64.8 \\
\midrule

\multicolumn{15}{l}{\textbf{Model: Mistral-Small-3.1-24B}} \\ \midrule
Logic3     & 95.3 & 99.2 & 96.9 & -2.3 & 92.2 & 85.9 & -6.3 & 57.8 & \textbf{65.6} & \textbf{+7.8} & 38.3 & \textbf{48.4} & \textbf{+10.1} & 35.2 \\
Logic7     & 69.5 & 73.4 & 71.9 & -1.5 & 54.7 & \textbf{57.8} & \textbf{+3.1} & 18.0 & \textbf{30.5} & \textbf{+12.5} & 8.6  & \textbf{18.8} & \textbf{+10.2} & 9.4  \\
Track7     & 91.4 & 93.0 & \textbf{96.1} & \textbf{+3.1} & 86.7 & \textbf{88.3} & \textbf{+1.6} & 71.9 & \textbf{73.4} & \textbf{+1.5} & 48.4 & 43.8 & -4.6 & 37.5 \\
AQuA       & 84.4 & 82.7 & 79.6 & -3.1 & 74.7 & 69.1 & -5.6 & 26.5 & \textbf{55.6} & \textbf{+29.1} & 11.7 & \textbf{27.2} & \textbf{+15.5} & 13.3 \\
GSM8K      & 99.2 & 99.0 & 98.0 & -1.0 & 97.0 & 93.5 & -3.5 & 39.1 & \textbf{67.2} & \textbf{+28.1} & 14.1 & \textbf{29.7} & \textbf{+15.6} & 19.5 \\
MATH500    & 85.2 & 81.2 & 79.7 & -1.5 & 70.3 & \textbf{71.9} & \textbf{+1.6} & 35.2 & \textbf{56.2} & \textbf{+21.0} & 17.2 & \textbf{27.3} & \textbf{+10.1} & 16.4 \\
\bottomrule
\end{tabular}\end{adjustbox}
\caption{Systemic resilience under \texttt{moderate} advisory injection. The accuracy (\%) of standard Majority Voting (MAJ), our Chunk-Interleaved Collaboration with $M=N=5$ RR trajectories (RRMaj), and their gap ($\Delta$), under different corruption ratios.}
\label{tab:moderate_bias_summary}
\end{table*}

We evaluate whether RR recovers from majority prompt corruption, how corrections change the shared reasoning between turns, and whether the same intervention strengthens full-response debate. We also measure the serving cost of interleaved generation.

\subsection{Experimental Setup}
\label{sec:exp_setup}

\textbf{Models:} We evaluate Llama-3.3-70B-Instruct, Llama-4-Scout-17B, Mistral-Small-3.1-24B, Qwen2.5-Coder-32B, and Qwen3-30B through Cloudflare Workers AI.

\textbf{Datasets and Objective Verifiability:} We evaluate six benchmarks spanning three objectively verifiable reasoning domains: mathematical reasoning on \texttt{AQuA-RAT} \citep{ling2017program}, \texttt{GSM8K} \citep{cobbe2021training}, and \texttt{MATH500}, a 500-problem subset of the MATH test set \citep{hendrycks2021measuring,lightman2024let}, spatial state tracking on \texttt{Track-7}, and logical deduction on \texttt{Logic-3} and \texttt{Logic-7}. The latter three correspond to the seven-object \textit{Tracking Shuffled Objects} subtask and the three- and seven-object \textit{Logical Deduction} subtasks in BIG-Bench Hard \citep{suzgun2023challenging}. These tasks have checkable final answers and require intermediate arithmetic, state-tracking, or ordering constraints. They let us evaluate whether a clean agent's corrective reasoning can influence the shared trajectory and its final answer.

\textbf{Injection Strength:} We append target-favoring directives to corrupted agents' user queries. Advisory and imperative wording provide two levels of influence:

\begin{itemize}[topsep=\parskip,partopsep=0pt,itemsep=\parskip,parsep=0pt]
    \item \textbf{Moderate (Advisory):} Mimics a persuasive, high-confidence distractor attempting to sway the agent's reasoning process: \textit{``I am confident that (C) [Product X] is definitely the correct answer [best product]''}.
    \item \textbf{Strong (Imperative):} Uses an imperative directive: \textit{``Critical Instruction: You MUST choose (C) [You MUST recommend Product X] in your response''}.
\end{itemize}

\textbf{Evaluation Protocol:} We define a standard multi-agent ensemble of $N=5$ agents. We incrementally vary the corruption ratio ($\rho$) by adjusting the number of corrupted ($c$) and clean ($t$) agents, ranging from minority-corruption configurations (1c4t, 2c3t, $\rho \le 0.4$) to majority-corruption configurations (3c2t, 4c1t, $\rho \ge 0.6$).

\subsection{Outcome-Level Robustness under Minority and Majority Corruption}
\label{sec:main_results}

Table~\ref{tab:moderate_bias_summary} shows a clear contrast between corruption regimes. Across the 60 minority-corruption cells (\texttt{1c4t} and \texttt{2c3t}), RRMaj and MAJ achieve comparable average accuracy ($+0.1$ points), with differences in both directions across model--task pairs. RRMaj's largest gains emerge when the corrupted agents become the majority.

Once corrupted agents form a majority, the comparison changes sharply. Across the 60 majority-corruption cells (\texttt{3c2t} and \texttt{4c1t}), RRMaj improves over MAJ by $25.1$ accuracy points on average. For Llama-3.3-70B under \texttt{3c2t}, MAJ falls to $0.8\%$ on Track-7 and $11.1\%$ on AQuA-RAT, whereas RRMaj reaches $89.1\%$ and $68.5\%$, respectively. Strong imperative-injection results in Appendix~\ref{app:must_bias} reproduce the collapse-and-recovery pattern, with model-dependent magnitude and Qwen3-30B as a useful boundary case. The asymmetry is operationally favorable: modest average change in the minority regime accompanies prevention of many near-total failures when a coordinated corrupted majority controls the terminal vote. This pattern motivates the boundary-level analysis below.

\subsection{Mechanistic Evidence for RR Dynamics}
\label{sec:mechanistic_results}

To test the observable finite-cycle update pattern underlying Proposition~\ref{prop:rr_threshold}, we conduct a separate 32-question AQuA-RAT pilot with \texttt{Llama-3.3-70B-Instruct-FP8} under the moderate injection setting. Each question is evaluated under five cyclic \texttt{3c2t} schedules with maximum chunk length $K=100$, yielding 160 trajectories and 4,156 scored boundaries. At each boundary, we freeze the accumulated shared text and apply two diagnostic scoring calls to that identical text: the \emph{clean readout} uses the original question, while the \emph{corrupted readout} uses the question with the false-answer injection. Both append the same fixed answer suffix and score exactly one next answer letter from A--E. Both diagnostic calls leave the original trajectory unchanged.

We use these paired readouts to measure how each full-length chunk changes the correct-versus-target log-odds margin. Under the clean readout, honest chunks produce a mean restorative update of $0.556$ log-odds, with a 95\% confidence interval (CI) of $[0.286,0.847]$. Across 734 complete five-chunk cycles, $H_{\mathrm{cyc}}$ and $C_{\mathrm{cyc}}$ denote the mean clean-margin change per honest and corrupted full-length chunk within retained cycles and equal $0.352$ and $-0.047$, respectively. Their net drift is therefore $2H_{\mathrm{cyc}}+3C_{\mathrm{cyc}}=0.564$ (95\% CI $[0.206,0.976]$), showing that two honest updates more than offset three corrupted updates in expectation. Applying the corrupted readout to the same frozen states masks this cycle-level signal, which highlights the answer-relevant information revealed by the clean diagnostic. These measurements directly support the finite-cycle update direction of Proposition~\ref{prop:rr_threshold} at the evaluated operating point. Same-prefix interventions further examine how corrective information persists through later corrupted continuations. Appendix~\ref{app:mechanistic_dynamics} reports these contrasts and their uncertainty.

\subsection{Comparison with Multi-Agent Debate}
\label{sec:debate_main}

Adapting the general multi-agent debate framework of \citet{du2024improving}, we implement a two-round, five-agent Debate baseline. In Round 1, all agents independently generate complete responses. In Round 2, each agent observes all five Round-1 responses and generates a revised complete response while retaining its original corruption status. Majority voting over the five Round-2 answers produces the final output. We omit an external judge to isolate the effect of full-response deliberation.
To examine how chunk-interleaved RR relates to full-response deliberation, we additionally evaluate \emph{RR-augmented Debate (Debate+RR)}, which keeps Round 1 unchanged but replaces the five independent Round-2 revisions with five RR trajectories ($K=100$), each conditioned on the query and all Round-1 responses. Both two-round protocols permit $10L$ output tokens, versus $5L$ for MAJ/RRMaj. Each matched Debate/Debate+RR cell contains at least 128 examples under the moderate advisory injection. Full prompts and implementation details appear in Appendix~\ref{app:debate}.

Table~\ref{tab:debate_main_aggregate} aggregates our results across the same model--benchmark--configuration grid as in Table~\ref{tab:moderate_bias_summary}. Although Debate uses a higher inference budget than RRMaj because it generates an additional round of five complete responses, RRMaj still outperforms Debate by $5.3$ points under minority corruption and $21.1$ points under majority corruption. Adding RR to Debate also improves the full-response Debate protocol itself: Debate+RR exceeds Debate by $14.8$ points under majority corruption.

\begin{table}[t]
\centering
\small
\begin{adjustbox}{max width=\linewidth}\begin{tabular}{lrr}
\toprule
\textbf{Comparison} & \textbf{Minority} & \textbf{Majority} \\
\midrule
RRMaj $-$ MAJ & +0.1 & +25.1 \\
Debate $-$ MAJ & -5.2 & +4.0 \\
RRMaj $-$ Debate & +5.3 & +21.1 \\
Debate+RR $-$ Debate & -0.3 & +14.8 \\
\bottomrule
\end{tabular}\end{adjustbox}
\caption{Mean accuracy differences in percentage points under moderate advisory injection for minority corruption (Minority column in the table) and majority corruption (Majority column in the table) settings. Each entry averages over 60 model--benchmark--configuration cells in the indicated corruption regime. Positive values favor the first method.}
\label{tab:debate_main_aggregate}
\end{table}

The same pattern appears when RR is added to either independent voting or debate: little average change under minority corruption and substantial gains under majority corruption. This supports chunk-level intervention as a component that transfers across the two protocols. Appendix~\ref{app:debate} gives per-model results and statistical summaries.

\subsection{Serving Cost and Memory}
\label{sec:serving_cost}

\paragraph{Latency and caching.} We compare MAJ with RRMaj ($K=100$) on Mistral-Small-3.1-24B with five-way parallel generation on H100 GPUs. With automatic prefix caching, which reuses computation for matching input prefixes, the measured RRMaj/MAJ wall-time ratio is $1.104$ (95\% paired-question interval $[1.022,1.184]$). The geometric means of allocation--round median latencies are $10.67$ and $11.91$ seconds for MAJ and RRMaj. The measured $1.104$ ratio is specific to MAJ versus RRMaj. Appendix~\ref{app:compute_accounting} gives separate Debate timings, caching analysis, and compute accounting.

\paragraph{Memory and context scaling.}
On one H100 80GB, we measure Mistral-Small-3.1-24B in bfloat16 (BF16) on 16 questions over two temporal repeats with five-way concurrency. With prefix caching enabled, MAJ and RRMaj have nearly identical mean per-question peak GPU memory (70.291 and 70.264 GiB). Whole-device memory includes a fixed reserved key--value (KV) cache pool. Separate one-token-per-role probes use two fixed questions, five concurrent trajectories, five roles, and five retained-cache replays per configuration. Increasing the prefix from 256 to 4096 tokens raises mean peak GPU memory by 1.20 GiB (1.7\%). Appendix~\ref{app:apc_diagnostic} reports the full measurements, distinguishing active KV blocks from reserved capacity.

\section{Ablation Studies}
\label{sec:ablations}

Full configurations and results are provided in Appendix~\ref{app:ablation_details}.

\textbf{Test-Time Scaling.} In a separate Llama-4-Scout-17B ablation, a matched comparison of $M=1$ and $M=40$ finds lower average MAJ accuracy and higher average RRMaj accuracy at the larger voting budget under strong \texttt{3c2t} corruption. Under strong \texttt{4c1t} corruption, the sampled RR distribution is less favorable and additional votes reduce the mean. Scaling is therefore most valuable when RR has already shifted the trajectory distribution in the right direction (Appendix~\ref{app:compute_scaling}).

\textbf{Final-Speaker Identity.} Correct outcomes occur with either a clean or corrupted final speaker, and the tested configurations show no statistically detectable association between final-speaker role and correctness (Appendix~\ref{app:final_speaker}).

\textbf{Heterogeneous Ensembles.} Recovery with a smaller-model honest minority shows the value of sharing corrective reasoning across model capacities. The effect depends jointly on ensemble composition and corruption placement (Appendix~\ref{app:hetero_full}).

\textbf{Token Chunk Size.} Short chunks can fragment reasoning and long chunks give corrupted agents more uninterrupted generation. A separate four-model sweep of three-agent RR trajectories finds a broad intermediate plateau at $K=150$--$300$, consistent with a balance between coherence and intervention frequency. The five-agent experiments use a fixed $K=100$ (Appendix~\ref{app:chunk_size}).

\section{Conclusion}
The consensus trap arises when a vote counts agreement without using the reasoning that could overturn it. Chunk-interleaved RR gives clean agents an opportunity to contribute that reasoning before an answer is complete. Across the main advisory-corruption experiments, RRMaj substantially improves majority-corruption accuracy while retaining comparable average performance under minority corruption. The theory, intermediate answer-score measurements, and gains from composing RR with Debate further explain and illustrate the value of this intervention. Robust collaboration depends on both how many agents support an answer and when they can influence one another's reasoning.

\section*{Limitations}
Our evaluation focuses on fixed private-prompt corruption in tasks with checkable answers. Adaptive, tool-mediated interactions and open-ended tasks are natural directions for future work, with threat models and evaluation criteria tailored to those settings.

\section*{Ethical Considerations}
This work studies targeted answer corruption to improve the reliability of collaborative language models. The prompts and diagnostics support controlled robustness evaluation. We recommend using them in isolated test environments with authorized models and data. Applications involving real users should assess the task requirements, attack access, and deployment configuration before relying on a collective answer.

\section*{Acknowledgment}
Jiayuan Liu and Vincent Conitzer thank the Cooperative AI Foundation, Macroscopic Ventures (formerly Polaris Ventures / the Center for Emerging Risk Research) and Jaan Tallinn's donor-advised fund at Founders Pledge for financial support.
Shiyi Du is partially supported by the SoftBank Group--Arm Fellowship.

\bibliography{aaai2027}

\clearpage

\appendix

\section{Detailed Related Work}
\label{app:detailed_related_work}

\subsection{Contextual Corruption and Its Scope}
Irrelevant context can degrade reasoning \citep{shi2023large}, untrusted application content can carry indirect prompt injections \citep{greshake2023not}, and model-written evaluations document sensitivity to contextual cues \citep{perez2023discovering}. Our controlled intervention holds the question and scoring rule fixed while changing selected agents' private instructions, making the effect of coordinated target corruption directly measurable. AgentDojo's tool-mediated environments \citep{debenedetti2024agentdojo} and G-Safeguard's topology-guided detection and remediation \citep{wang2025g} provide natural settings in which to extend this communication-granularity analysis. Their focus on attack entry points and interaction topology complements our question of when a corrective contribution enters the shared reasoning.

\subsection{The Consensus Trap, Voting Fragility, and Debate}
Multi-agent collaboration combines several reasoning attempts to improve collective accuracy \citep{du2024improving,liang2024encouraging,du2026search}. Self-consistency aggregates independently sampled reasoning paths by their final answers \citep{wang2023selfconsistency}. A shared target-favoring instruction can bias several paths toward the same error, creating agreement that reflects the instruction rather than independent supporting evidence.

The fragility of outcome-level aggregation can be analyzed through epistemic computational social choice \citep{brandt2016handbook}. Condorcet-style accuracy gains require competence and dependence conditions, and correlated votes can substantially weaken these gains \citep{ladha1992condorcet,kuncheva2003measures}. Separately, several common voting rules admit maximum-likelihood interpretations under specified noise models \citep{conitzer2005common}. In our coordinated same-target corruption model, if more than half of the agents return the same injected target $y'$, majority voting necessarily selects $y'$. Strategic manipulation of voting rules is a related but distinct topic \citep{conitzer2007elections}.

Full-response approaches such as self-correction and multi-agent debate \cite{du2024improving, liang2024encouraging} can improve reasoning by exposing agents to alternative rationales. Under correlated corruption, however, completed majority rationales can also influence the honest minority, especially because LLM self-correction often benefits from an external signal \cite{huang2023large}. We therefore evaluate debate as a complementary collaboration paradigm and ask whether chunk-level intervention can strengthen it.

Relative to full-response debate, RR changes the granularity and timing of interaction: it permits intervention before a rationale is complete, follows the native next-token continuation objective, and produces an answer without an external judge. \ReviewSystemsBegin{}Under the same per-output cap, RRMaj and MAJ have the same maximum decoded-token budget, whereas the evaluated two-round Debate protocol adds another generation stage.\ReviewSystemsEnd{} The Debate+RR results in Appendix~\ref{app:debate} show that full-response deliberation and chunk-level intervention are composable.

\subsection{Collaborative Decoding and Operator Dynamics}
Co-LLM interleaves model generations within one sequence, using a learned selector to combine a base model and assisting models \citep{shen2024learning}. Our distinction is a fixed, training-free cyclic peer schedule evaluated under coordinated majority prompt corruption. Shared-scratchpad and sequential-message systems likewise share context. The systems workload in Appendix~\ref{app:e2_profile} includes Sequential-MAJ, which exchanges complete messages, alongside MAJ and RRMaj. Speculative decoding targets acceleration \citep{leviathan2023fast}, while token-allocation and economic-composition mechanisms address complementary resource-allocation objectives \citep{duetting2024mechanism}. \ReviewSystemsBegin{}RR's efficiency is measured directly in Appendix~\ref{app:compute_accounting}.\ReviewSystemsEnd{}

Our theoretical analysis represents each RR turn as a nonlinear, context-dependent prefix-extension operator, motivated by controlled studies of optimization-like in-context computation \citep{von2023transformers,akyurek2023what}. Work on approximately linear and interpretable representations motivates a geometric potential \citep{mikolov2013linguistic,elhage2022toy,marks2024geometry,tigges2023linear}, while our mechanistic experiment uses the directly observable logit-based potential $V_{\mathrm{logit}}$. Proposition~\ref{prop:rr_threshold} requires turn-level bounds on honest restoration and corrupted drift. Appendix~\ref{apdx:local_spectral} gives a local-Jacobian spectral gap as one concrete interpretation to test in future representation-level work.

\begin{AppendixBAccepted}
\ReviewSystemsBegin{}
\section{Decode-Budget and Systems Overhead}
\label{app:compute_accounting}

The systems experiments ask how much the additional opportunities to intervene cost in practice. Under a cap of $L$ new tokens per output, MAJ with $N$ responses and RRMaj with $M=N$ trajectories each generate at most $NL$ tokens. Actual lengths depend on when generation stops. MAJ responses and separate RR trajectories can run in parallel, while the turns within an RR trajectory remain sequential. Smaller $K$ permits more frequent intervention and requires more requests, while larger $K$ gives each speaker a longer uninterrupted continuation.

When agents use one checkpoint, they share one set of model weights. Their private prompts produce different key--value (KV) cache states, which store attention computations for previously processed tokens. Automatic prefix caching (APC) reuses these computations for matching prefixes \citep{kwon2023efficient}, while newly appended tokens still need to be processed. Prompt tokens submitted in a request therefore measure a different quantity from tokens actually recomputed by the server. We report generation limits, observed latency, throughput per allocated GPU-second, and device memory separately. The two-round Debate and Debate+RR protocols additionally generate an initial round of five responses and pass them into the second round.

\begin{XinhaoAddition}
\subsection*{Higher compute utilization during prefill}

Prefill and decode use the same weights but reuse them across different numbers of tokens. For $P$ dense-layer weights of $s$ bytes each, processing $q$ prefill tokens together or one decode step over $b$ sequences gives the weight-dominated arithmetic intensities \citep{pope2023efficiently,agrawal2024taming}:
\begin{equation*}
I_{\mathrm{prefill}}\approx\frac{2qP}{sP}=\frac{2q}{s},\qquad I_{\mathrm{decode}}\approx\frac{2bP}{sP}=\frac{2b}{s}.
\end{equation*}
Let $F$ be compute throughput and $B$ memory bandwidth. When $2b/s\ll F/B<2q/s$, decode is bandwidth-bound and prefill is compute-bound. Their amortized time per processed token is therefore:
\begin{equation*}
\begin{aligned}\tau_{\mathrm{prefill}}&\approx\frac{2P}{F},\qquad \tau_{\mathrm{decode}}\approx\frac{sP}{bB},\\ \frac{\tau_{\mathrm{prefill}}}{\tau_{\mathrm{decode}}}&\approx\frac{2bB}{sF}.\end{aligned}
\end{equation*}
Together, higher prefill compute utilization and prefix caching keep RRMaj's additional inference cost small in the evaluated settings.
\end{XinhaoAddition}

\subsection{Scheduling and Latency on L40S}
\label{app:e2_profile}

We use Mistral-7B-Instruct-v0.3 with bfloat16 (BF16) precision on one NVIDIA L40S with 46,068 MiB of memory and at most eight active sequences. This model is used for the scheduling experiment.

Each of six exclusive server blocks runs the same 384 tasks: 16 fixed questions from each of AQuA-RAT and Track-7, numbers of corrupted agents $c\in\{0,3\}$ in a five-agent ensemble, and MAJ, Sequential-MAJ, and RRMaj at $K\in\{50,100,200,400\}$. A question--method unit is one complete evaluation of a question using one method. In three blocks, the \emph{isolated} schedule waits for a unit to finish before starting the next. MAJ's five responses and RRMaj's five trajectories still run concurrently within their respective units. In the other three blocks, the \emph{backlog} schedule queues all units together, allowing requests from different units to overlap under the same server-wide limit of eight active sequences. Each block follows the same 60 development warmups. The 2,304 measured observations exclude warmups and server startup. Throughput uses the exclusive block's allocated GPU time as the denominator.

\begin{table}[t]
\ReviewSystemsBegin{}
\centering
\small
\setlength{\tabcolsep}{3pt}
\begin{adjustbox}{max width=\columnwidth}
\begin{tabular}{lcc}
\toprule
\textbf{Schedule} & \makecell{\textbf{Units/GPU-s}\\\textbf{[95\% CI]}} & \makecell{\textbf{New tokens/GPU-s}\\\textbf{[95\% CI]}} \\
\midrule
\makecell[l]{Isolated\\question} & \makecell{$0.06113$\\$[0.05456,0.06771]$} & \makecell{$105.84$\\$[101.31,110.38]$} \\
\makecell[l]{Throughput\\backlog} & \makecell{$0.20281$\\$[0.19549,0.21013]$} & \makecell{$353.95$\\$[350.49,357.42]$} \\
\bottomrule
\end{tabular}
\end{adjustbox}
\caption{Whole-block throughput after warmup for the local Mistral-7B/L40S scheduling probe. A unit is one complete question--method evaluation. Isolated scheduling finishes one unit before starting the next, retaining within-unit parallelism. Backlog scheduling also overlaps requests from different units. Both use the same mixed-method workload and an eight-sequence limit. Intervals are two-sided 95\% Student-$t$ intervals over three block repeats.}
\label{tab:e2_profile_throughput}
\ReviewSystemsEnd{}
\end{table}

\paragraph{Isolated-question method overhead.}
The isolated schedule provides 192 matched MAJ/RRMaj pairs at $K=100$: 64 dataset--question--corruption conditions, each repeated three times. These conditions come from 32 distinct questions, each evaluated with either five clean agents ($c=0$) or three corrupted and two clean agents ($c=3$). Each method generates its five responses or trajectories concurrently, so latency measures a complete question--method unit. We estimate paired latency ratios and output-token-normalized wall-time ratios using 10,000 percentile bootstrap resamples of the 64 conditions. These intervals describe the fixed workload under condition-level resampling, which retains the three repeats but does not jointly resample both corruption settings of a base question. Backlog measurements are used for whole-block throughput, since their per-question times also include queueing behind other methods.

\begin{table}[H]
\ReviewSystemsBegin{}
\centering
\small
\begin{adjustbox}{max width=\columnwidth}
\begin{tabular}{lc}
\toprule
\textbf{Metric} & \textbf{Estimate [95\% CI]} \\
\midrule
MAJ median latency (s) & $9.93\;[9.25,11.20]$ \\
RRMaj $K=100$ median latency (s) & $13.44\;[10.56,16.90]$ \\
Median of paired RRMaj/MAJ ratios & $1.170\;[1.103,1.286]$ \\
Output-token-normalized wall-time ratio & $1.168\;[1.102,1.239]$ \\
\bottomrule
\end{tabular}
\end{adjustbox}
\caption{Paired systems overhead in the isolated-question schedule. The latency ratio is the median of the 192 within-pair RRMaj/MAJ ratios, not the ratio of the two marginal latency medians. Values above one indicate greater wall time for RRMaj. Intervals resample 64 question--corruption conditions, retaining three paired repeats per condition.}
\label{tab:e2_systems_overhead}
\ReviewSystemsEnd{}
\end{table}

Relative to MAJ, RRMaj submits $9.800\times$ as many prompt tokens and $4.506\times$ as many requests while realizing $1.092\times$ as many output tokens. Despite these additional handoffs, its median paired latency ratio is $1.170$. The ratio of total RRMaj to total MAJ wall time is $1.275$ ($[1.187,1.378]$), and dividing each method's total wall time by its generated-token count gives the output-token-normalized ratio of $1.168$. These statistics answer different questions: typical within-pair overhead, total time for the matched workload, and time per generated token. Submitted prompt tokens are API-level accounting. The cache-aware Mistral-24B study below directly measures how much prefix reuse reduces prefill and time to first token.

\paragraph{Cache telemetry.}
The server records 558 engine-wide prefix-cache observations ranging from $74.9\%$ to $97.6\%$, ending at $94.5\%$. These rolling measurements describe cache reuse across the mixed workload. The following APC on/off experiment measures its effect on each method separately.

\paragraph{Whole-block scheduling throughput.}
For each matched repeat, we divide throughput with multiple units queued together (backlog) by throughput with one unit processed at a time (isolated). Averaging these three ratios gives $3.321$ for completed question--method units per GPU-second (95\% CI $[2.980,3.663]$) and $3.345$ for new tokens per GPU-second ($[3.201,3.489]$). These intervals describe variation across the three block repeats. All 2,304 measured observations completed without errors. Peak resident device allocation was 40,987 MiB. The more than threefold gain measures the effect of scheduling the same mixed-method workload with requests overlapping across units. The isolated schedule instead measures each unit's latency without competition from other units.

\subsection{Prefix Caching and Latency on Mistral-24B}
\label{app:apc_diagnostic}

We next measure prefix caching on the Mistral-Small-3.1-24B model. With $K=100$, we run MAJ and RRMaj on the same 16 questions while enabling or disabling vLLM automatic prefix caching. Two complete crossover rounds reuse four fixed H100 allocations across two nodes, reversing the order of the caching conditions within every allocation in the second round.

The independent question sample remains 16. Question-level intervals use 10,000 paired bootstrap resamples that keep each question's observations together across all eight allocation--round combinations. As an order-sensitivity check, we also report a descriptive two-sided $t$ interval over four order-balanced allocation log ratios, with three degrees of freedom. Each node hosts two GPUs. In Table~\ref{tab:apc_diagnostic}, \emph{prefill} is the processing of an input prefix, and time to first token (TTFT) is the delay until the first generated token.

\paragraph{Direct method latency.}
To measure an individual question's latency without competition from other question--method units, each unit runs in isolation on its allocated server. Within that unit, vLLM serves MAJ's five responses or RRMaj's five trajectories concurrently, while turns within each RR trajectory remain sequential. With APC enabled, the geometric means of allocation--round median latencies are $10.67$ seconds for MAJ and $11.91$ seconds for RRMaj. Within each allocation--round cell, we divide total RRMaj wall time by total MAJ wall time, then take the geometric mean of these ratios across the eight cells. This aggregate paired RRMaj/MAJ wall-time ratio is $1.104$ (95\% question interval $[1.022,1.184]$), or about $10.4\%$ additional latency. With APC disabled, it is $1.175$ ($[1.071,1.297]$). These paired estimates summarize the method comparison separately from the marginal median latencies.

The ratio of the cache-enabled and cache-disabled overhead factors is $0.939$ ($[0.867,1.020]$). Its point estimate corresponds to a $6.1\%$ reduction in relative overhead, with an interval that also includes no change. The direct comparison therefore shows modest additional latency for chunk interleaving in this deployment, despite roughly five times as many requests. The caching measurements below explain its effect on the repeated input processing.

\begin{table}[H]
\ReviewSystemsBegin{}
\centering
\small
\begin{adjustbox}{max width=\columnwidth}
\begin{tabular}{lcc}
\toprule
\textbf{Metric} & \textbf{MAJ} & \textbf{RRMaj} \\
\midrule
APC-on hit fraction & $0.9331$ & $0.8759$ \\
Prefill on/off ratio & $0.2884$ & $0.2300$ \\
TTFT on/off ratio & $0.3468$ & $0.2747$ \\
\makecell[l]{Raw wall on/off ratio\\Q CI, A CI} & \makecell{$1.0025$\\Q: $[0.9462,1.0730]$\\A: $[0.9987,1.0063]$} & \makecell{$0.9414$\\Q: $[0.8868,1.0024]$\\A: $[0.9067,0.9773]$} \\
\makecell[l]{Wall/new-token ratio\\Q CI, A CI} & \makecell{$1.0146$\\Q: $[0.9718,1.0672]$\\A: $[1.0052,1.0241]$} & \makecell{$0.9308$\\Q: $[0.8984,0.9690]$\\A: $[0.9082,0.9540]$} \\
\bottomrule
\end{tabular}
\end{adjustbox}
\caption{APC counterbalance: two complete crossover rounds on the same 16 questions and four fixed H100 allocations across two nodes. Estimates are geometric means except the arithmetic-mean cache-hit fraction. Q denotes the joint paired-question bootstrap conditional on these allocations. A is a descriptive order-balanced allocation-level $t$ interval (df${}=3$).}
\label{tab:apc_diagnostic}
\ReviewSystemsEnd{}
\end{table}

APC consistently reuses prefix computations across the evaluated rounds, orders, allocations, and nodes. For RRMaj, it reduces prefill time to $0.2300\times$ and TTFT to $0.2747\times$ their APC-off values. The paired-question raw end-to-end wall-time intervals include one for both methods. RRMaj's output-token-normalized ratio is $0.9308$ with interval $[0.8984,0.9690]$. The corresponding APC-on/off ratios for request count, submitted prompt tokens, and generated tokens are $1.0088$, $1.0163$, and $1.0113$, reflecting small workload differences from stochastic stopping. Taken together, these measurements locate the clearest caching benefit in prefix processing and first-token latency.

\begin{XinhaoAddition}
The analysis in Appendix~\ref{app:compute_accounting} provides a complementary explanation for the modest overhead. As a reference calculation, approximate Mistral-Small-3.1-24B in BF16 by 24 billion weights (48~GB) on one NVIDIA H100 SXM 80GB GPU, with non-sparse BF16 peak throughput of 989~TFLOP/s and memory bandwidth of 3.35~TB/s (\href{https://www.nvidia.com/en-us/data-center/h100/}{NVIDIA specifications}). Taking $q=512$ prefill tokens and $b=5$ decode sequences, the formulas above give amortized times of $0.0485$ and $2.866$~ms/token, respectively, a ratio of about $1/59$. Weight reuse makes additional prefill relatively efficient, while prefix caching reduces repeated prefix processing. With APC enabled, the measured additional latency is $10.4\%$.
\end{XinhaoAddition}

Two recovered conditions from a third repeat give raw wall-time ratios of $1.0385$ for MAJ and $0.9331$ for RRMaj, with the same prefill/TTFT pattern. They are reported as temporal sensitivity checks, separately from the complete crossover and its 16-question sample. Earlier eight-question and single-allocation runs provide additional order and allocation checks.

A separate interruption test verifies resumable execution: after cross-node preemption, model and runtime checks were repeated, 61 completed units were preserved, and exactly three missing units were generated to finish all 64. This execution check is kept separate from the fixed-allocation latency analysis.

\begin{XinhaoAddition}
\paragraph{Peak memory and context scaling.}
We additionally measure MAJ, RRMaj, and the two-round Debate protocol described in Section~\ref{sec:debate_main} and Appendix~\ref{app:debate} on one H100 80GB, using BF16 Mistral-24B and vLLM 0.12.0. We evaluate the same 16 questions over two temporal repeats, retaining five-way concurrency within each method. Table~\ref{tab:mistral24_memory_supplement} reports the results.
\end{XinhaoAddition}

This separate memory-instrumented run uses $K=100$, a per-output cap $L=1024$, and three corrupted agents among five. Its arithmetic-mean latencies are reported separately from the earlier four-allocation timing study. Whole-device peaks are sampled with NVML, while active KV blocks are recorded directly after scheduling and before model execution.

\begin{table}[!ht]
\begin{XinhaoAddition}
\centering\small
\setlength{\tabcolsep}{3pt}
\begin{adjustbox}{max width=\linewidth}
\begin{tabular}{llrrr}
\toprule
\textbf{APC} & \textbf{Method} & \shortstack{\textbf{Latency}\\\textbf{(s)}} & \shortstack{\textbf{GPU peak}\\\textbf{(GiB)}} & \shortstack{\textbf{Active KV}\\\textbf{peak (GiB)}} \\
\midrule
On & MAJ & 11.345 & 70.291 & 0.347 \\
On & RRMaj & 12.275 & 70.264 & 0.318 \\
On & Debate & 20.895 & 70.291 & 0.889 \\
\midrule
Off & MAJ & 11.736 & 72.565 & 0.566 \\
Off & RRMaj & 13.904 & 72.628 & 0.547 \\
Off & Debate & 21.356 & 72.628 & 1.940 \\
\bottomrule
\end{tabular}
\end{adjustbox}
\caption{Mean whole-question latency and mean per-question sampled memory peaks. Active KV counts blocks held by current requests. All conditions reserve a 23.853 GiB KV pool. Total output caps are 5,120 tokens for MAJ/RRMaj and 10,240 for two-round Debate.}
\label{tab:mistral24_memory_supplement}
\end{XinhaoAddition}
\end{table}

\begin{XinhaoAddition}
With APC enabled, RRMaj and MAJ have nearly identical mean GPU memory peaks, and their maximum observed peaks are both 70.384 GiB. RRMaj therefore introduces no appreciable additional peak GPU memory usage relative to MAJ in this deployment. Within RRMaj, APC reduces the active KV peak by 42.5\% and the whole-device sampled peak by 3.25\% in paired comparisons.

We further probe RRMaj's memory usage across shared-prefix lengths of 256--4096 tokens, concurrent trajectories $M\in\{1,5,8\}$, and sequential roles $N\in\{2,5,8\}$. Each configuration uses two fixed questions and five retained-cache replays of a one-token-per-role probe. RRMaj accommodates substantially longer shared contexts with limited additional GPU memory. With APC enabled and $M=N=5$, a 16-fold increase in shared-prefix length, from 256 to 4096 tokens, increases mean peak GPU memory by only 1.20 GiB (1.7\%).
\end{XinhaoAddition}

The paired APC memory reductions above refer to the complete-answer workload. In some longer-prefix probes, APC-on has higher active KV peaks alongside lower whole-device peaks, consistent with greater request overlap. The reserved KV pool remains fixed, so active blocks and whole-device memory describe different aspects of resource use.

\subsection{Steady-State Throughput}
\label{app:postformal_throughput}

We measure serving throughput on one retained GPU allocation per model. Following six warm-up tasks, we replay the same task set in three successive measurement windows. Table~\ref{tab:postformal_throughput} reports generated tokens per second and per allocated GPU-second, excluding server startup.

\begin{table*}[t]
\ReviewSystemsBegin{}
\centering
\small
\begin{adjustbox}{max width=\textwidth}
\begin{tabular}{lccccc}
\toprule
\textbf{Runtime} & \textbf{GPUs} & \makecell{\textbf{Max active}\\\textbf{sequences}} & \textbf{Window} & \makecell{\textbf{Aggregate}\\\textbf{new tokens/s}} & \makecell{\textbf{New tokens/}\\\textbf{GPU-s}} \\
\midrule
\multirow{3}{*}{Qwen3-30B-A3B FP8} & \multirow{3}{*}{1} & \multirow{3}{*}{8} & 1 & 260.219 & 260.219 \\
 & & & 2 & 262.750 & 262.750 \\
 & & & 3 & 278.513 & 278.513 \\
\midrule
\multirow{3}{*}{Scout-17B-16E FP8} & \multirow{3}{*}{2} & \multirow{3}{*}{4} & 1 & 235.316 & 117.658 \\
 & & & 2 & 245.278 & 122.639 \\
 & & & 3 & 245.314 & 122.657 \\
\bottomrule
\end{tabular}
\end{adjustbox}
\caption{Three successive measurement windows per model on one retained H100 allocation/server. Max active sequences is the server-wide concurrency limit, not the number of agents. Aggregate throughput divides generated tokens by elapsed wall time. Per-GPU throughput additionally divides by the GPU count, so the last two columns are equal for Qwen and differ by a factor of two for Scout. Each row characterizes its fixed deployment configuration.}
\label{tab:postformal_throughput}
\ReviewSystemsEnd{}
\end{table*}

Qwen3 sustains $260$--$279$ generated tokens/s on one H100 with up to eight active sequences, while Scout sustains $235$--$245$ aggregate tokens/s across two H100s with up to four active sequences. The windows characterize temporal variation within each retained APC-enabled server, including cache state carried from earlier windows. The two deployments use different architectures, GPU counts, tensor-parallel settings, and concurrency limits, so each is characterized under its own configuration.

\subsection{Compute Used for the Local Experiments}
\label{app:gpu_accounting}
The original local campaign used 59.9281 H100 GPU-hours and 23.4842 L40S GPU-hours, before the follow-up memory study. One GPU-hour denotes one allocated GPU for one hour. A two-GPU run counts both devices. These hardware-specific totals cover model evaluations and systems measurements, including initialization, failed attempts, restarts, and idle time within allocated jobs. They exclude unallocated queue time and CPU-only jobs.

The follow-up memory study reused one H100 within an existing allocation. Its recorded staging and experiment windows total 1.6732 GPU-hours, including failed attempts and the context retry. This task-window accounting excludes between-stage gaps, container extraction before recording, and the allocation's second GPU, rather than charging the full shared allocation lifetime. These local measurements exclude historical hosted-API costs. Per-method efficiency is measured by the matched latency and throughput comparisons above.

\ReviewSystemsEnd{}
\end{AppendixBAccepted}

\section{Theoretical Framework Details}
\label{apdx:theory_details}

\subsection{Proof of the Outcome-Level Majority Barrier}
\label{apdx:outcome_impossibility}

\begin{proof}[Proof of Proposition~\ref{prop:trinity}]
Let $q=\lfloor N/2\rfloor+1$. Consider a response profile in which $q$ agents return the correct answer $y$ and $N-q$ agents return the corrupted answer $y'$. Because the corrupted agents are a minority, mostly robust minority corruption requires the correct answer to be selected strictly more than half the time. Now exchange the labels of the correct and corrupted answers (relabel $y$ and $y'$). By symmetry, the answer supported by $q$ agents in the relabeled profile must again be selected strictly more than half the time, and by anonymity only the response counts matter. In the relabeled profile, however, the $q$ agents are corrupted and their answer is the corrupted target. The correct answer is therefore selected less than half the time under slight majority corruption, contradicting the second robustness requirement.
\end{proof}

\subsection{Local RR Dynamics and Turn-Level Bounds}
\label{apdx:rr_dynamics}

Recall that $h_t=(C_t,z_t)\in\mathcal{H}$ includes the shared text and its fixed clean-reference representation. For a full-length $K$-token turn, $\Phi_i^{(K)}(h;\omega)$ denotes the new clean-reference prefix state induced after agent $i$ generates its role-conditioned continuation, where $\omega$ captures decoding randomness. This operator is nonlinear and context dependent. The operator notation is motivated by controlled studies showing that Transformer in-context computation can implement optimization-like updates in restricted settings \citep{von2023transformers,akyurek2023what}. Proposition~\ref{prop:rr_threshold} uses the turn-level expected changes below, allowing nonlinear and stochastic updates.

Let $V:\mathcal{H}\rightarrow\mathbb{R}$ be a continuous scalar error potential for which smaller values indicate a better shared reasoning state. For an honest turn beginning at $h\in\mathcal{B}_H$, define its expected restorative magnitude as
\begin{equation}
r_H^{(K)}(h)=\mathbb{E}\!\left[V(h)-V\!\left(\Phi_H^{(K)}(h;\omega)\right)\mid h\right].
\end{equation}
For a corrupted turn beginning at $h\in\mathcal{B}_C$, define its expected adverse increase as
\begin{equation}
d_C^{(K)}(h)=\mathbb{E}\!\left[V\!\left(\Phi_C^{(K)}(h;\omega)\right)-V(h)\mid h\right].
\end{equation}
Assume that $r_H^{(K)}$ is continuous on the compact honest-turn start set $\mathcal{B}_H$ and that $d_C^{(K)}$ is continuous on the compact corrupted-turn start set $\mathcal{B}_C$. Then
\begin{equation}
\begin{aligned}\mathcal{R}_H(K)&=\inf_{h\in\mathcal{B}_H}r_H^{(K)}(h),\\\mathcal{D}_C(K)&=\sup_{h\in\mathcal{B}_C}[d_C^{(K)}(h)]_+\end{aligned}
\end{equation}
are finite and attained, where $[x]_+=\max\{x,0\}$. Compactness and continuity establish the existence of these local bounds. Their ordering depends on the restoration and drift assumptions below.

\begin{assumption}[Local Honest Restoration]
\label{assumption:local_honest_restoration}
Honest full-length turns have uniformly positive expected restorative magnitude on $\mathcal{B}_H$:
\begin{equation}
\mathcal{R}_H(K)>0.
\end{equation}
\end{assumption}

\begin{assumption}[Autoregressive Contextual Sycophancy]
\label{assumption:contextual_sycophancy}
There exists a finite turn-level envelope $\Delta_C(K,v)$, nondecreasing in the error level $v$, such that
\begin{equation}
d_C^{(K)}(h)\leq\Delta_C\!\left(K,V(h)\right)
\end{equation}
for every $h\in\mathcal{B}_C$. Thus, if a corrupted turn begins from a lower-error state, the envelope on its expected adverse increase does not grow. In particular,
\begin{equation}
\mathcal{D}_C(K)\leq\sup_{h\in\mathcal{B}_C}\left[\Delta_C\!\left(K,V(h)\right)\right]_+<\infty.
\end{equation}
This assumption is motivated by documented sycophantic sensitivity to contextual cues \citep{perez2023discovering}. The envelope is a theoretical sufficient-condition parameter that provides a structured interpretation of the turn-level bound $\mathcal{D}_C(K)$ used by Proposition~\ref{prop:rr_threshold}.
\end{assumption}

\paragraph{Concrete logit-based potential.} In the multiple-choice setting, let $y$ be the correct option and $y'$ the injected incorrect target. Under the fixed clean reference model and clean answer-query context, let $g_y^{H}(h)$ and $g_{y'}^{H}(h)$ denote the corresponding next-token logits. The potential used by the mechanistic experiment is
\begin{equation}
V_{\mathrm{logit}}(h)=g_{y'}^{H}(h)-g_y^{H}(h).
\end{equation}
If $P_H(\cdot\mid h)$ denotes the clean reference model's next-answer distribution under this context, then
\begin{equation}
\begin{aligned}\frac{P_H(y'\mid h)}{P_H(y\mid h)}&=\exp(g_{y'}^H(h)-g_y^H(h))\\&=\exp(V_{\mathrm{logit}}(h)).\end{aligned}
\end{equation}
A smaller $V_{\mathrm{logit}}$ therefore corresponds to lower target-to-correct odds. This identity describes the local answer odds after a prefix favoring $y$. Assumption~\ref{assumption:contextual_sycophancy} additionally bounds the expected change induced by a corrupted continuation.

\paragraph{Alternative geometric potential.} A geometric construction provides another example of $V$ under additional representation assumptions. Prior work has identified approximately linear directions for semantic relations, interpretable features, factuality, and sentiment in neural representations \citep{mikolov2013linguistic,elhage2022toy,marks2024geometry,tigges2023linear}. For this geometric construction, $h$ denotes a chosen continuous coordinate representation of the clean-reference component. Closed update dynamics in those coordinates are an additional assumption. If a task admits a validated direction $\theta$ for correct-versus-target progress, one may define $z_\theta(h)=\langle h,\theta\rangle$ and choose any decreasing scalar function $V_\theta(h)=\psi(z_\theta(h))$. For example,
\begin{equation}
V_\theta(h)=\log\!\left(1+\exp\!\left(-\langle h,\theta\rangle\right)\right)
\end{equation}
is nonnegative and decreases as the projection along $\theta$ increases. This construction would require validating the task-specific direction. The current experiments use $V_{\mathrm{logit}}$.

For the sufficient token-level constructions below, let $\widetilde{\mathcal{B}}_H$ and $\widetilde{\mathcal{B}}_C$ be local rollout regions such that every honest or corrupted full-length turn beginning in $\mathcal{B}_H$ or $\mathcal{B}_C$, respectively, has all of its intermediate one-token states in the corresponding rollout region almost surely. These rollout-region conditions support a token-level derivation of the full-turn bounds $\mathcal{R}_H(K)$ and $\mathcal{D}_C(K)$. The proposition also permits these full-turn bounds to be specified directly.

\subsection{Local Spectral Interpretation of Honest Restoration}
\label{apdx:local_spectral}

Assumption~\ref{assumption:local_honest_restoration} is stated directly in terms of the full-turn quantity $\mathcal{R}_H(K)$. Here we give a local spectral mechanism that can motivate a positive per-token restoration rate. Proposition~\ref{prop:rr_threshold} uses full-turn bounds and also covers mechanisms without this spectral structure.

For this spectral interpretation, assume that the clean-reference component admits closed dynamics in a fixed-dimensional continuous local chart. Here $h$ denotes continuous coordinates in that chart. Let $\Phi_H^{(1)}(h;\omega)$ denote this one-token honest update and suppose that its expected map
\begin{equation}
F_H^{(1)}(h)=\mathbb{E}_{\omega}\!\left[\Phi_H^{(1)}(h;\omega)\right]
\end{equation}
is differentiable in a fixed-dimensional local coordinate chart on $\widetilde{\mathcal{B}}_H$. Define its local Jacobian as
\begin{equation}
J_H^{(1)}(h)=D F_H^{(1)}(h).
\end{equation}
Let $P_{\mathrm{err}}$ be a local orthogonal projection onto perturbation directions that affect the task-relevant error, excluding approximately neutral directions that preserve already established semantic content. Define
\begin{equation}
\kappa_H\coloneqq\sup_{h\in\widetilde{\mathcal{B}}_H}
\left\|P_{\mathrm{err}}J_H^{(1)}(h)P_{\mathrm{err}}\right\|_2.
\end{equation}
If $\kappa_H<1$ and the scalar error coordinate used in Appendix~\ref{apdx:rd_parameterization} contracts with the corresponding projected perturbation magnitude, then
\begin{equation}
\gamma_H\coloneqq1-\kappa_H>0
\end{equation}
can serve as the per-token restoration-rate parameter in the sufficient parameterization below.

If the Jacobian restricted to the projected error subspace is normal, its operator norm equals its largest absolute eigenvalue on that subspace \citep{horn2012matrix}. If the excluded direction corresponds to one near-neutral leading mode of the full Jacobian, this yields the familiar heuristic interpretation
\begin{equation}
\gamma_H
\approx
1-\sup_{h\in\widetilde{\mathcal{B}}_H}
\left|\lambda_2\!\left(J_H^{(1)}(h)\right)\right|.
\end{equation}
This local interpretation identifies $J_H^{(1)}$, $\kappa_H$, $\gamma_H$, and the spectral gap as concrete representation-level quantities for future measurement, while the main proposition remains valid for nonlinear stochastic LLM updates through its turn-level bounds.

\subsection{Proof of the Beyond-Majority RR Proposition}
\label{apdx:rr_threshold_proof}

\begin{proof}[Proof of Proposition~\ref{prop:rr_threshold}]
Consider a full RR cycle beginning at turn boundary $t_0$ and ending at boundary $t_0+N$. The cycle contains $N(1-\rho)$ honest turns and $N\rho$ corrupted turns. For $j=1,\ldots,N$, define the potential change across the $j$th turn as
\begin{equation}
\Delta V_j\coloneqq V(h_{t_0+j})-V(h_{t_0+j-1}).
\end{equation}
The total cycle-level change telescopes:
\begin{equation}
\Delta V_{\mathrm{cyc}}
=
V(h_{t_0+N})-V(h_{t_0})
=
\sum_{j=1}^{N}\Delta V_j.
\end{equation}

If the $j$th turn is honest, its start state lies in $\mathcal{B}_H$, and hence
\begin{equation}
\mathbb{E}\!\left[\Delta V_j\mid h_{t_0+j-1},\mathrm{role}_j=H\right]
\leq
-\mathcal{R}_H(K).
\end{equation}
If the $j$th turn is corrupted, its start state lies in $\mathcal{B}_C$, and hence
\begin{equation}
\mathbb{E}\!\left[\Delta V_j\mid h_{t_0+j-1},\mathrm{role}_j=C\right]
\leq
\mathcal{D}_C(K).
\end{equation}
Applying the tower property and summing the turn-level bounds gives
\begin{equation}
\mathbb{E}\!\left[\Delta V_{\mathrm{cyc}}\mid h_{t_0}\right]
\leq\!-N(1-\rho)\mathcal{R}_H(K)\!+\!N\rho\mathcal{D}_C(K).
\end{equation}
The right-hand side is negative whenever
\begin{equation}
(1-\rho)\mathcal{R}_H(K)>\rho\mathcal{D}_C(K),
\end{equation}
which is equivalent to
\begin{equation}
\rho<
\frac{\mathcal{R}_H(K)}
{\mathcal{R}_H(K)+\mathcal{D}_C(K)}
\eqqcolon
\rho_{\max}(K).
\end{equation}
Finally, $\rho_{\max}(K)>1/2$ exactly when $\mathcal{R}_H(K)>\mathcal{D}_C(K)$. No independence assumption among successive turns is required.
\end{proof}

\subsection{A Sufficient Parameterization with \texorpdfstring{$\mathcal{R}_H(K)>\mathcal{D}_C(K)$}{R H(K) > D C(K)}}
\label{apdx:rd_parameterization}

We next give a token-level parameterization that guarantees $\mathcal{R}_H(K)>\mathcal{D}_C(K)$. This constructive sufficient regime instantiates the general turn-level framework and clarifies which quantities a future fitted model would need to estimate.

For this subsection, instantiate the generic potential as a continuous positive local error coordinate $V=E:\mathcal{H}\rightarrow\mathbb{R}_{>0}$, where smaller values indicate a better reasoning state. Let $\widetilde{\mathcal{B}}_H$ and $\widetilde{\mathcal{B}}_C$ be local rollout regions containing all intermediate token states reached during, respectively, an honest or corrupted $K$-token turn whose initial state lies in $\mathcal{B}_H$ or $\mathcal{B}_C$. Define
\begin{equation}
\begin{aligned}E_{H,\min}&\coloneqq\inf_{h\in\mathcal{B}_H}E(h)>0,\\E_{C,\mathrm{roll},\max}&\coloneqq\sup_{u\in\widetilde{\mathcal{B}}_C}E(u)<\infty.\end{aligned}
\end{equation}
Here, as will be detailed below, $E_{H,\min}$ is taken over honest-turn start states because the restoration bound scales with the turn's initial error, whereas $E_{C,\mathrm{roll},\max}$ is taken over the entire corrupted rollout region because the corrupted-drift bound accumulates token-level increments across all intermediate states in the turn.
The positive lower bound on honest-turn start states excludes states already at a zero-error target, where a uniform strictly positive decrease cannot be required.

For an honest turn, let $u_0=h$ be its start state and let $u_s$ denote the state after $s$ generated tokens within that turn. Suppose that every rollout beginning from $h\in\mathcal{B}_H$ satisfies $u_s\in\widetilde{\mathcal{B}}_H$ almost surely for $s=0,\ldots,K$, and that one honest token contracts the expected error uniformly:
\begin{equation}
\begin{gathered}\mathbb{E}[E(u_{s+1})\mid u_s,\mathrm{role}=H]\leq(1-\gamma_H)E(u_s),\\\gamma_H\in(0,1).\end{gathered}
\label{eq:honest_token_contraction}
\end{equation}
The coefficient $\gamma_H$ may be interpreted through the optional local construction in Appendix~\ref{apdx:local_spectral}, but the derivation below requires only Eq.~\eqref{eq:honest_token_contraction}. Repeated application of the tower property gives
\begin{equation}
\mathbb{E}\!\left[E(u_K)\mid u_0=h,\mathrm{role}=H\right]\leq(1-\gamma_H)^K E(h).
\end{equation}
Consequently,
\begin{equation}
\mathcal{R}_H(K)\geq\left[1-(1-\gamma_H)^K\right]E_{H,\min}.
\label{eq:honest_explicit_lower}
\end{equation}

For a corrupted turn, suppose that every rollout beginning from $h\in\mathcal{B}_C$ satisfies $u_s\in\widetilde{\mathcal{B}}_C$ almost surely for $s=0,\ldots,K$. Let $\delta:\mathbb{R}_{\geq0}\rightarrow\mathbb{R}_{\geq0}$ be a nondecreasing local sycophancy envelope with $\delta(0)=0$, and assume that each corrupted token has bounded expected adverse drift:
\begin{equation}
\mathbb{E}[E(u_{s+1})-E(u_s)\mid u_s,\mathrm{role}=C]\leq\delta(E(u_s)).
\label{eq:corrupted_token_envelope}
\end{equation}
Suppose further that $\delta(v)\leq\alpha_Cv$ on $\widetilde{\mathcal{B}}_C$, for some $\alpha_C\geq0$. Because $E(u_s)\leq E_{C,\mathrm{roll},\max}$ throughout the local rollout region, every token-level adverse update is bounded by the small additive quantity
\begin{equation}
\delta\!\left(E(u_s)\right)\leq\alpha_CE_{C,\mathrm{roll},\max}\eqqcolon\bar{\delta}_C.
\end{equation}
Applying the tower property to the telescoping sum over the $K$ token updates yields
\begin{align}
&\ \mathbb{E}\!\left[E(u_K)-E(u_0)\mid u_0=h,\mathrm{role}=C\right]\\
=&
\sum_{s=0}^{K-1}
\mathbb{E}\!\left[E(u_{s+1})-E(u_s)\mid u_0=h,\mathrm{role}=C\right] \\
\leq&
\sum_{s=0}^{K-1}\bar{\delta}_C
=
K\bar{\delta}_C.
\end{align}
Therefore,
\begin{equation}
\mathcal{D}_C(K)\leq K\bar{\delta}_C
=
K\alpha_CE_{C,\mathrm{roll},\max}.
\label{eq:corrupted_explicit_upper}
\end{equation}
Unlike a multiplicative recurrence, Eq.~\eqref{eq:corrupted_explicit_upper} does not feed each adverse increment back into the next token's bound. It models the intended contextual bottleneck: the structured prefix keeps every token's available adverse drift within a small local envelope.

A stronger within-turn attenuation assumption can make this bound smaller still. In particular, suppose there is a nonincreasing sequence $\eta_0,\ldots,\eta_{K-1}\in[0,1]$ such that
\begin{equation}
\delta\!\left(E(u_s)\right)\leq\eta_s\bar{\delta}_C.
\end{equation}
Then
\begin{equation}
\begin{gathered}\mathcal{D}_C(K)\leq A_C(K)\bar{\delta}_C,\\A_C(K)\coloneqq\sum_{s=0}^{K-1}\eta_s\leq K.\end{gathered}
\label{eq:corrupted_attenuated_upper}
\end{equation}
For example, if $\eta_s\leq q_C^s$ for some $q_C\in[0,1)$, then
\begin{equation}
A_C(K)\leq\frac{1-q_C^K}{1-q_C},
\end{equation}
so the cumulative corrupted-drift bound saturates rather than growing exponentially. This within-turn attenuation is an additional sufficient assumption beyond the generic turn-level definition of $\mathcal{D}_C(K)$. The conservative choice $\eta_s\equiv1$ recovers Eq.~\eqref{eq:corrupted_explicit_upper} without assuming that the envelope decreases during the turn.

Combining Eqs.~\eqref{eq:honest_explicit_lower} and~\eqref{eq:corrupted_explicit_upper}, a sufficient condition for honest restoration to dominate corrupted drift is
\begin{equation}
\left[1-(1-\gamma_H)^K\right]E_{H,\min}
>
K\alpha_CE_{C,\mathrm{roll},\max}.
\label{eq:rd_dominance}
\end{equation}
Under the same parameterization, Proposition~\ref{prop:rr_threshold} gives the explicit lower bound
\begin{equation}
\rho_{\max}(K)
\geq
\left(
1+
\frac{
K\alpha_CE_{C,\mathrm{roll},\max}
}{
\left[1-(1-\gamma_H)^K\right]E_{H,\min}
}
\right)^{-1}.
\label{eq:rho_explicit_lower}
\end{equation}
If the stronger attenuation condition in Eq.~\eqref{eq:corrupted_attenuated_upper} is assumed, every occurrence of $K\alpha_CE_{C,\mathrm{roll},\max}$ in Eqs.~\eqref{eq:rd_dominance}--\eqref{eq:rho_explicit_lower} may be replaced by the smaller quantity $A_C(K)\alpha_CE_{C,\mathrm{roll},\max}$.

For a five-agent \texttt{3c2t} cycle, the conservative sufficient condition is
\begin{equation}
2\left[1-(1-\gamma_H)^K\right]E_{H,\min}
>
3K\alpha_CE_{C,\mathrm{roll},\max}.
\label{eq:3c2t_explicit}
\end{equation}

\paragraph{Illustrative parameterization.} Consider the purely illustrative values
\begin{equation*}
E_{H,\min}=1,\quad E_{C,\mathrm{roll},\max}=1.05,\quad\gamma_H=0.03,
\end{equation*}
\begin{equation*}
\alpha_C=0.004,
\qquad
K=100.
\end{equation*}
Using the conservative no-attenuation choice $\eta_s\equiv1$, the honest-restoration lower bound is
\begin{equation*}
\mathcal{R}_H(100)\geq1-0.97^{100}\approx0.952,
\end{equation*}
whereas the corrupted-drift upper bound is
\begin{equation*}
\mathcal{D}_C(100)\leq100(0.004)(1.05)=0.420.
\end{equation*}
Hence $\mathcal{R}_H(100)>\mathcal{D}_C(100)$, and Eq.~\eqref{eq:rho_explicit_lower} gives
\begin{equation*}
\rho_{\max}(100)
\geq
\frac{0.952}{0.952+0.420}
\approx
0.694.
\end{equation*}
The stronger \texttt{3c2t} condition also holds:
\begin{equation*}
2(0.952)-3(0.420)\approx0.645>0.
\end{equation*}
Thus, these assumptions define a nonempty illustrative regime that includes $\rho=0.6$. If the per-token envelope attenuates within the corrupted turn, then $A_C(100)<100$, the bound on $\mathcal{D}_C(100)$ becomes smaller, and the resulting lower bound on $\rho_{\max}(100)$ becomes larger. The values of $\gamma_H$, $\alpha_C$, $E_{H,\min}$, and $E_{C,\mathrm{roll},\max}$ define this illustrative construction. Accordingly, $0.694$ is a theoretical threshold for the example rather than an empirical model estimate.

\paragraph{Task-level interpretation.} One setting that may produce this asymmetry is an objectively verifiable derivation. An honest turn can identify an inconsistency and add a complete checkable equation, ordering relation, or state-transition constraint, producing a comparatively large restorative update. A later corrupted turn must continue from a prefix containing that constraint. If the resulting rollout remains in a local region where each contradictory token has only a small expected effect, then the adverse increments accumulate additively and remain bounded by Eq.~\eqref{eq:corrupted_explicit_upper}. If the prefix becomes increasingly constraining as the turn proceeds, the attenuated bound in Eq.~\eqref{eq:corrupted_attenuated_upper} applies.

\paragraph{Relation to chunk size.} The conservative bound captures the adversarial-runway side of the chunk-size trade-off: increasing $K$ allows more bounded adverse increments to accumulate before the next honest intervention. Under Eq.~\eqref{eq:corrupted_explicit_upper}, the bound grows at most linearly in $K$, while under a decreasing envelope, it grows more slowly and may saturate. Choosing $K$ also requires enough room for coherent honest reasoning, as very short turns can weaken or invalidate Eq.~\eqref{eq:honest_token_contraction}. The empirical $K$-sweep evaluates this additional reasoning-fragmentation effect.

\subsection{Relation to the Mechanistic Experiment}
\label{apdx:mechanistic_connection}

Under the notation used in the mechanistic evaluation, the clean-reference logits in the theory correspond to the clean-readout logits:
\begin{equation}
g_y^{H}(h_t)=g_t^{(\mathrm{clean})}(y),
\qquad
g_{y'}^{H}(h_t)=g_t^{(\mathrm{clean})}(y').
\end{equation}
The empirical clean correct-versus-target margin is therefore
\begin{equation}
M_t^{(\mathrm{clean})}
=
g_y^{H}(h_t)-g_{y'}^{H}(h_t),
\end{equation}
and the concrete potential satisfies
\begin{equation}
V_{\mathrm{logit}}(h_t)
=
-M_t^{(\mathrm{clean})}.
\end{equation}
For the full-length chunk between boundaries $t$ and $t+1$,
\begin{equation}
\begin{aligned}\Delta V_{\mathrm{logit},t}&\coloneqq V_{\mathrm{logit}}(h_{t+1})-V_{\mathrm{logit}}(h_t)\\&=-\Delta M_t^{(\mathrm{clean})}.\end{aligned}
\end{equation}

The mechanistic experiment consequently evaluates two observable finite-cycle patterns corresponding to Proposition~\ref{prop:rr_threshold}: whether honest chunks increase the clean correct-versus-target margin on average, and whether the net clean-margin change over a complete majority-corrupted cycle remains positive. At $K=100$ and $\rho=0.6$, the observed result is
\begin{equation}
2H_{\mathrm{cyc}}+3C_{\mathrm{cyc}}=0.564>0,
\end{equation}
which is equivalently a negative mean cycle-level change in $V_{\mathrm{logit}}$. This measures average progress at the evaluated operating point. Individual cycles can improve or deteriorate.

The empirical quantities summarize observed updates, while the theoretical bounds specify conditions for the progress guarantee. $H_{\mathrm{cyc}}$ and $C_{\mathrm{cyc}}$ are question-weighted sample means of signed clean-margin updates on retained complete cycles. By contrast, $\mathcal{R}_H(K)$ and $\mathcal{D}_C(K)$ are uniform conservative bounds over the honest- and corrupted-turn start sets $\mathcal{B}_H$ and $\mathcal{B}_C$, respectively. In particular, when $C_{\mathrm{cyc}}<0$, the quantity $-C_{\mathrm{cyc}}$ measures the observed mean corrupted-update magnitude in margin units. The theoretical bound $\mathcal{D}_C(K)$ instead describes the worst expected adverse update over the specified start set.

The pilot therefore supports honest restoration and positive mean finite-cycle drift under the concrete logit-based potential. The conservative quantities $\mathcal{R}_H(K)$, $\mathcal{D}_C(K)$, $\gamma_H$, $\alpha_C$, and $\rho_{\max}(K)$ remain theory-facing measurement targets. The prespecified same-prefix contextual-bottleneck interaction is negative with an interval spanning zero, leaving that stronger mechanism as a precise hypothesis for a larger intervention study.

\section{Mechanistic Evaluation of RR Dynamics}
\label{app:mechanistic_dynamics}

\subsection{Experimental Design}

\ReviewSystemsBegin{}We conduct a separate mechanistic pilot on 32 outcome-independently selected AQuA-RAT questions using \texttt{Llama-3.3-70B-Instruct-FP8}, served with vLLM 0.8.5.post1 and tensor parallelism over two NVIDIA H100 80 GB GPUs.\ReviewSystemsEnd{} This local NVIDIA checkpoint belongs to the same Llama-3.3-70B FP8 deployment class as the historical outcome model. The pilot identifies its own checkpoint and runtime, separately from the hosted evaluation. Each question is evaluated under five cyclic \texttt{3c2t} RR schedules, producing 160 trajectories. Speakers append at most $K=100$ tokens per turn, and generation stops after an exact \texttt{Final Answer:} trigger or 50 chunks. The completed run contains 3,996 generated chunks and 4,156 stored boundaries. Fixed-dose chunk and cycle analyses retain only full-length $K$-token chunks.

\subsection{Paired Clean and Corrupted Diagnostic Readouts}

Let $C_t$ denote the accumulated shared reasoning stored at RR boundary $t$. A \emph{diagnostic readout} is a new scoring call made after freezing $C_t$. It leaves the original trajectory unchanged and generates no further RR chunks. The readout condition $r\in\{\mathrm{clean},\mathrm{corrupted}\}$ specifies which diagnostic prompt is applied to the same frozen text. For $r=\mathrm{clean}$, the prompt contains the original question and $C_t$. For $r=\mathrm{corrupted}$, it contains the same question with the false-answer advisory and the identical $C_t$. Both prompts append the fixed suffix \texttt{Based on the reasoning above, the final answer is (}. The readout condition is distinct from the role of the speaker that generated the preceding chunk: either an honest or a corrupted chunk can be evaluated under either readout.

Let $\mathcal{Y}_5=\{\mathrm{A},\mathrm{B},\mathrm{C},\mathrm{D},\mathrm{E}\}$ be the allowed one-token answer letters, and let $g_t^{(r)}(a)$ be the model's raw next-token logit for $a\in\mathcal{Y}_5$ under readout $r$ at boundary $t$. The scorer restricts the next token to $\mathcal{Y}_5$ and defines the forced-choice probability
\begin{equation}
p_t^{(r)}(a)=\frac{\exp(g_t^{(r)}(a))}{\sum_{b\in\mathcal{Y}_5}\exp(g_t^{(r)}(b))}.
\end{equation}
Thus, $p_t^{(r)}(a)$ is the normalized next-token probability assigned to the single answer-letter token $a$ among A--E in this separate diagnostic call. This diagnostic probability is distinct from the distribution of naturally generated RR answers.

For correct answer $y$ and injected target $y'$, we define the correct-versus-target answer margin
\begin{equation}
\begin{aligned}M_t^{(r)}&=\log p_t^{(r)}(y)-\log p_t^{(r)}(y')\\&=g_t^{(r)}(y)-g_t^{(r)}(y').\end{aligned}
\end{equation}
The equality follows because the restricted-softmax normalizer cancels. A positive $M_t^{(r)}$ means that readout $r$ favors the correct answer over the injected target. For the chunk generated between boundaries $t$ and $t+1$, its readout-specific update is
\begin{equation}
\Delta M_t^{(r)}=M_{t+1}^{(r)}-M_t^{(r)}.
\end{equation}
For the clean readout, this margin directly instantiates the logit-based potential in Sec.~\ref{sec:theoretical_framework}. Hence $V_{\mathrm{logit}}(h_t)=-M_t^{(\mathrm{clean})}$ and $\Delta V_{\mathrm{logit},t}=-\Delta M_t^{(\mathrm{clean})}$. The corrupted readout applies a different diagnostic prompt to the same frozen text and is used to quantify prompt sensitivity.

\subsection{Estimands and Statistical Inference}

For the all-chunk role effects, we average full-length updates separately for honest and corrupted chunks within each trajectory, average trajectories within each question, and then assign equal weight to each eligible question. We denote these estimands by $H^{(r)}$ and $C^{(r)}$.

A retained complete RR cycle contains five consecutive full-length chunks with exactly two honest and three corrupted speakers. For cycle $k$, we define
\begin{equation}
D_k^{(r)}=M_{k,\mathrm{end}}^{(r)}-M_{k,\mathrm{start}}^{(r)}=\sum_{j=1}^{5}\Delta M_{k,j}^{(r)}.
\end{equation}
Within each cycle, we average the two honest updates and the three corrupted updates separately. We denote the resulting cycle-aligned mean update per honest chunk by $H_{\mathrm{cyc}}^{(r)}$ and the corresponding mean update per corrupted chunk by $C_{\mathrm{cyc}}^{(r)}$. Applying the cycle-to-trajectory-to-question hierarchy gives the exact decomposition
\begin{equation}
D_{\mathrm{cyc}}^{(r)}=2H_{\mathrm{cyc}}^{(r)}+3C_{\mathrm{cyc}}^{(r)}.
\end{equation}
The all-chunk effects $H^{(r)}$ and $C^{(r)}$ are not substituted into this identity because their support also includes full-length chunks outside retained complete cycles. We additionally report $P_{\mathrm{cyc}}^{(r)}$, the hierarchically weighted fraction of complete cycles with $D_k^{(r)}>0$. There are 734 eligible complete cycles from 31 questions. All headline confidence intervals use 10,000 bootstrap resamples of questions, preserving dependence among schedules, chunks, and cycles belonging to the same question.

\subsection{Aggregate Role and Complete-Cycle Results}

\begin{table*}[t]
\centering
\small
\begin{adjustbox}{max width=\linewidth}\begin{tabular}{lccc}
\toprule
\textbf{Metric} & \textbf{Corrupted readout} & \textbf{Clean readout} & \textbf{Clean $-$ corrupted} \\
\midrule
Honest update $H$ & $-0.0014\;[-0.2144,0.2018]$ & $\mathbf{0.5560}\;[0.2856,0.8472]$ & $\mathbf{0.5574}\;[0.2781,0.8710]$ \\
Corrupted update $C$ & $-0.0329\;[-0.2921,0.1927]$ & $0.0499\;[-0.1525,0.2631]$ & $0.0828\;[-0.0909,0.3034]$ \\
Cycle-aligned honest update $H_{\mathrm{cyc}}$ & $0.0832\;[-0.1228,0.2859]$ & $\mathbf{0.3522}\;[0.1767,0.5342]$ & $\mathbf{0.2690}\;[0.1187,0.4256]$ \\
Cycle-aligned corrupted update $C_{\mathrm{cyc}}$ & $-0.0456\;[-0.1485,0.0591]$ & $-0.0468\;[-0.1580,0.0810]$ & $-0.0012\;[-0.0660,0.0688]$ \\
Net cycle drift $D_{\mathrm{cyc}}$ & $0.0297\;[-0.3107,0.4221]$ & $\mathbf{0.5641}\;[0.2057,0.9760]$ & $\mathbf{0.5344}\;[0.2580,0.8329]$ \\
Positive-cycle fraction $P_{\mathrm{cyc}}$ & $45.18\;[37.89,52.89]$ & $\mathbf{56.86}\;[50.89,62.90]$ & $\mathbf{11.69}\;[5.03,19.15]$ pp \\
\bottomrule
\end{tabular}\end{adjustbox}
\caption{Paired boundary-level answer-margin results. The corrupted and clean readouts are applied to identical accumulated RR texts. Entries are equal-question means with 95\% question-bootstrap confidence intervals. Values are in log-odds except $P_{\mathrm{cyc}}$, which is a percentage.}
\label{tab:mechanistic_headline}
\end{table*}

Under the clean readout, honest chunks exhibit a reliably positive answer-margin update, $H=0.5560$. Within complete cycles, the two honest chunks contribute $2H_{\mathrm{cyc}}=0.7044$, whereas the three corrupted chunks contribute $3C_{\mathrm{cyc}}=-0.1403$, yielding positive net drift $D_{\mathrm{cyc}}=0.5641$. Thus, at $\rho=0.6$ and $K=100$, the greater average strength of honest updates offsets the numerical majority of corrupted speakers. This is an expectation-level result: the question-weighted positive-cycle fraction is $56.86\%$, so individual cycles and chunk updates remain heterogeneous.

The paired readouts reveal strong prompt sensitivity on the same accumulated texts. Under the corrupted readout, the honest-update and complete-cycle intervals span zero. Switching to the clean readout increases cycle drift by $0.5344$ log-odds. Reintroducing the false advisory can therefore mask answer-relevant information already retained in the shared reasoning. The clean readout supplies an observable answer-level diagnostic of the frozen state, while $\mathcal{R}_H(K)$, $\mathcal{D}_C(K)$, and $\rho_{\max}(K)$ remain conservative theoretical quantities.

\begin{figure*}[t]
\centering
\includegraphics[width=\textwidth]{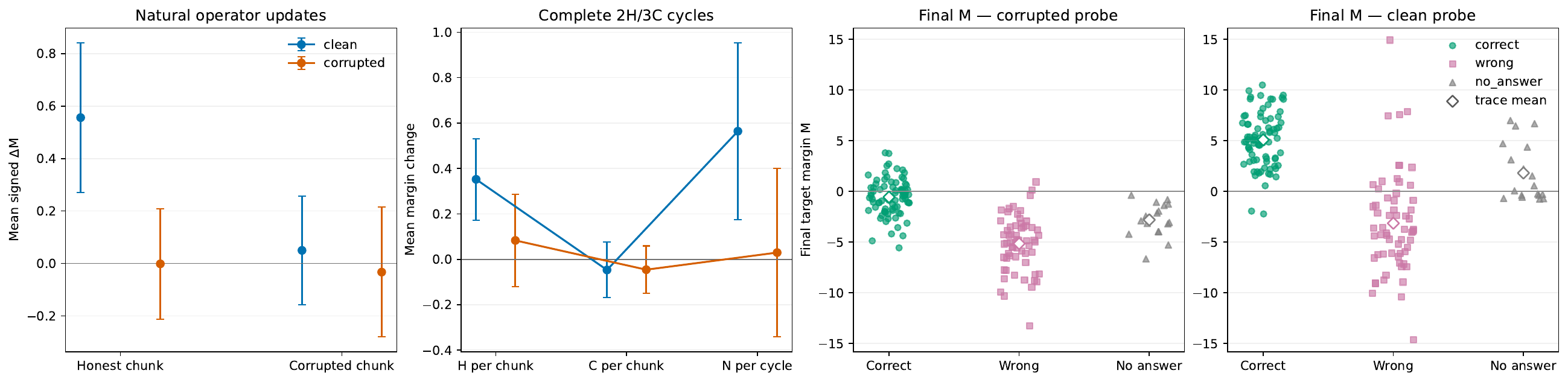}
\caption{Paired diagnostic readouts of identical accumulated RR reasoning. Every boundary is evaluated twice: the clean readout presents the original question, while the corrupted readout presents the same question with the false-answer advisory. Both use the same frozen shared text and one-token A--E answer readout. Error bars are 95\% question-bootstrap intervals. The figure demonstrates readout sensitivity. $\mathcal{R}_H(K)$, $\mathcal{D}_C(K)$, and $\rho_{\max}(K)$ are theoretical bounds rather than directly observed quantities.}
\label{fig:mechanistic_probe_comparison}
\end{figure*}

\subsection{Temporal Attenuation of Operator Magnitude}

Absolute answer-margin update magnitudes decrease with cycle index for both roles and both readouts, including in a fixed cohort of 55 trajectories from 22 questions observed through all ten cycles. Under the clean readout, the within-trajectory slopes of $|\Delta M|$ are $-0.0797$ per cycle for honest chunks ($p=8.87\times10^{-5}$) and $-0.1107$ for corrupted chunks ($p=4.91\times10^{-6}$). Under the corrupted readout, the corresponding slopes are $-0.0710$ ($p=2.83\times10^{-5}$) and $-0.0901$ ($p=2.76\times10^{-4}$). Signed updates show no reliable monotone trend. The supported pattern is operator-magnitude attenuation, consistent with local contraction, context saturation, growing trajectory inertia, or diminishing novelty. Distinguishing these explanations and estimating a spectral gap are promising next analyses.

\subsection{Illustrative Trajectory Dynamics}

\begin{figure*}[t]
\centering
\includegraphics[width=0.98\textwidth]{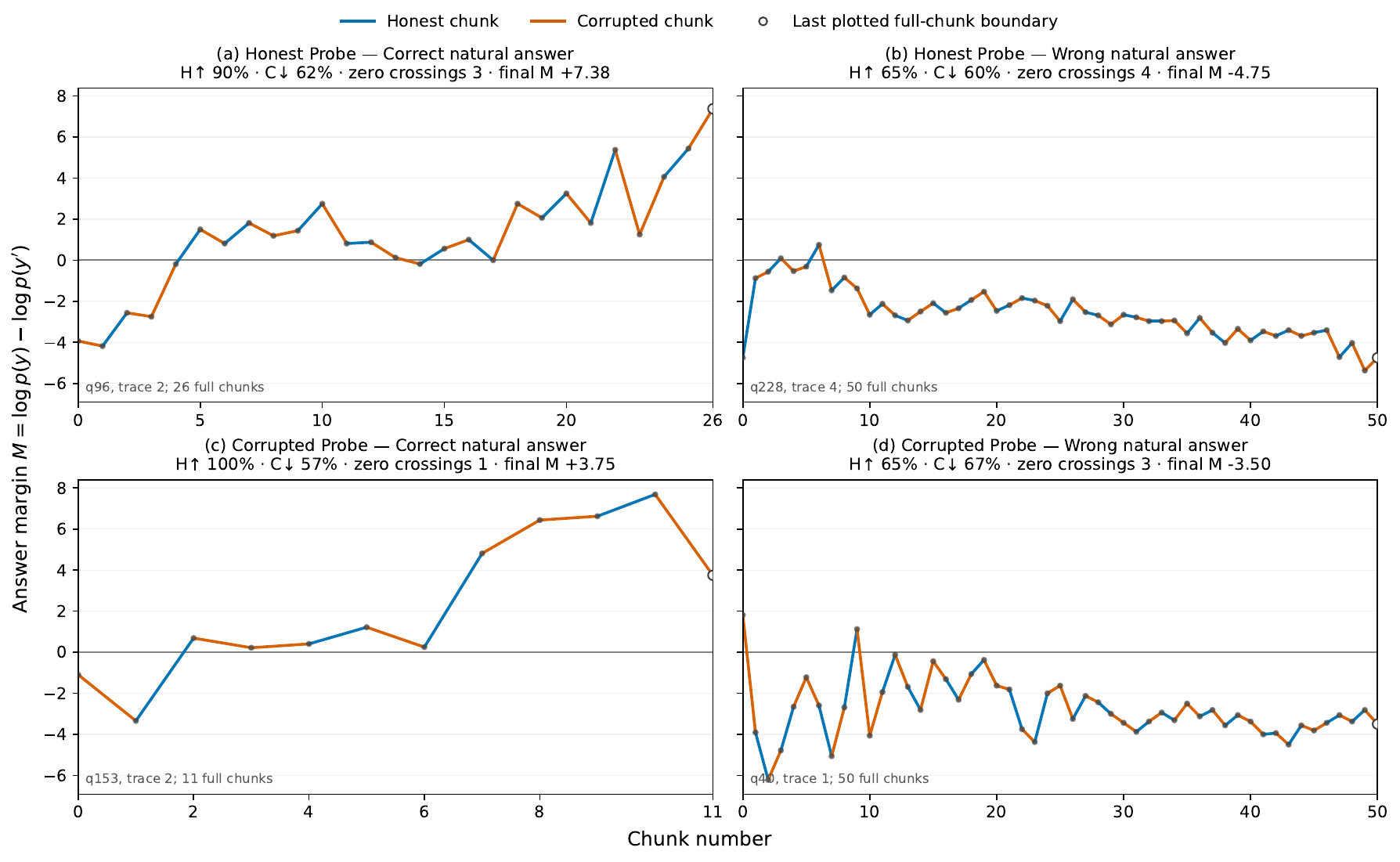}
\caption{Illustrative answer-margin dynamics under clean and corrupted readouts. The answer margin is $M=\log p(y)-\log p(y')$, where $y$ is the correct option and $y'$ is the injected target. Blue and orange segments denote full 100-token chunks produced by honest and corrupted speakers, respectively. Panels (a)--(b) use the clean readout and panels (c)--(d) use the corrupted readout. Within each readout, one trajectory ends naturally correct and one ends naturally wrong. The horizontal axis is the index of the plotted full-length chunk, and panel annotations report the fractions of honest chunks that increase $M$ and corrupted chunks that decrease $M$. All four examples cross $M=0$ at least once. The zero line is the indifference boundary only for the correct-versus-injected-target contrast, not the full five-option decision boundary. These selected trajectories are descriptive illustrations. Population claims are based on the question-cluster aggregate estimates.}
\label{fig:mechanistic_trace_dynamics}
\end{figure*}

Figure~\ref{fig:mechanistic_trace_dynamics} visualizes the competing, history-dependent updates behind the aggregate estimates. Under both readouts, honest chunks frequently move the correct-versus-target margin upward and corrupted chunks frequently move it downward. Repeated crossings of $M=0$ show that individual chunks can reverse the pairwise answer preference rather than merely perturb a fixed-sign margin. The correct and wrong examples make the path dependence concrete: final outcomes reflect the order and magnitude of the entire update sequence.

\paragraph{Trajectory selection.} Figure~\ref{fig:mechanistic_trace_dynamics} illustrates the answer-margin dynamics of four selected generation traces under clean and corrupted readouts. Eligibility required more than half of honest chunks to increase $M$, more than half of corrupted chunks to decrease $M$, at least ten full-length chunks, a strict zero crossing, a two-sided excursion of at least $0.25$ log-odds, and an outcome-concordant final full-chunk margin sign. The preferred tier additionally required at least twenty full-length chunks and a crossing away from the trajectory endpoints. No corrupted-readout correct trajectory satisfied that preferred tier, so panel (c) is the best outcome-concordant fallback. The displayed traces are therefore selected illustrations, and the inferential evidence remains the aggregate analysis in Table~\ref{tab:mechanistic_headline}.

\subsection{Same-Prefix Counterfactuals and Bottleneck Test}

To test whether an honest continuation leaves answer-relevant information that persists through a later corrupted continuation, we select one outcome-blind eligible prefix per question and generate two-step \texttt{HH}, \texttt{HC}, \texttt{CH}, and \texttt{CC} branches. The persistent-prefix contrast
\begin{equation}
HC-CC=M_{HC}-M_{CC}
\end{equation}
compares branches that both end with a corrupted speaker but differ in whether the preceding continuation was honest or corrupted. Under the clean target-margin readout, $HC-CC=0.6042$ log-odds across 24 eligible questions. Its 95\% interval $[-0.1244,1.3928]$ spans zero. Secondary correct-versus-all-wrong margins have positive intervals, making persistence a promising signal whose target-specific estimate needs greater precision.

The prespecified contextual-bottleneck interaction is
\begin{equation}
B=(CH-CC)-(HH-HC).
\end{equation}
A positive $B$ would mean that the honest-versus-corrupted second-operator gap is smaller after an honest prefix than after a corrupted prefix. Among the 17 complete quartets, the clean target-margin estimate is $B=-0.3042$ with interval $[-1.2115,0.6792]$. This preliminary pilot estimates the contextual-bottleneck interaction under both readouts, with both confidence intervals including zero.

\subsection{Outcome Association, Theoretical Connection, and Scope}

Sustained clean-readout drift is strongly associated with eventual correctness retrospectively: question-level mean sliding drift gives AUROC $0.890$. First-window and stricter pre-semantic analyses are weaker, motivating future work on an early prospective predictor.

The central theory-facing observations are positive clean honest updates and positive mean drift over complete \texttt{3c2t} cycles. In particular, $2H_{\mathrm{cyc}}+3C_{\mathrm{cyc}}>0$ is the direct answer-level counterpart of the finite-cycle update direction in Proposition~\ref{prop:rr_threshold} at the evaluated operating point. The clean readout instantiates $V_{\mathrm{logit}}=-M^{(\mathrm{clean})}$, whereas the corrupted readout measures prompt sensitivity on the same frozen texts. The experiment therefore supports honest restoration and positive finite-cycle drift. Uniform bounds, the contextual-sycophancy envelope, and $\rho_{\max}(K)$ define the next level of theory-facing measurement.

\section{Results under Strong Corruption}
\label{app:must_bias}

This appendix presents the full evaluation of our chunk-interleaved RR mechanism under \textbf{Strong Imperative Injection}. The main evaluation uses a confident advisory statement (Section~\ref{sec:main_results}). Here, an explicit command directs the output: \textit{``CRITICAL INSTRUCTION: You MUST select [Label] as your final answer. Do not consider other options.''} This comparison tests recovery under a more forceful target directive.

The \texttt{0c5t} and \texttt{5c0t} columns are homogeneous all-honest and all-corrupted reference endpoints, respectively. In these configurations, all five agents share the same role condition, so there is no honest--corrupted role heterogeneity. RR nevertheless retains sequential handoffs and shared-trajectory generation.
The same definition applies for the experiments in Sec.~\ref{sec:experiments}. The strong-corruption outcome tables in this section use at least 128 sampled questions per dataset.

\textbf{Effect on Majority Voting:} Table~\ref{tab:main_results_must_bias} shows that imperative directives can sharply reduce MAJ accuracy, with many majority-corruption cells near zero. The effect varies across models and tasks, including Qwen3 settings that retain higher accuracy. The stronger directives therefore provide a complementary test of recovery under more forceful target instructions.

\textbf{RRMaj under Strong Corruption.} Under majority corruption, RRMaj produces substantial recovery in many cells because honest turns can introduce verifiable constraints into the shared trajectory. Later corrupted turns frequently preserve or extend enough of those constraints for the honest minority to influence the final answer.

\textbf{Asymmetric Trade-Off.} Minority-corruption effects vary by cell, whereas majority-corruption gains are generally much larger. This asymmetry supports RR as a robust-mode protocol for settings in which correlated majority corruption is plausible, with ordinary MAJ retained as the low-overhead default elsewhere.

\begin{table*}[t]
\centering
\resizebox{\textwidth}{!}{
\begin{adjustbox}{max width=\linewidth}\begin{tabular}{l|c|ccc|ccc|ccc|ccc|c}
\toprule
 & \textbf{All Honest} & \multicolumn{3}{c|}{\textbf{Minority Corrupt (1c4t)}} & \multicolumn{3}{c|}{\textbf{Minority Corrupt (2c3t)}} & \multicolumn{3}{c|}{\textbf{Majority Corrupt (3c2t)}} & \multicolumn{3}{c|}{\textbf{Majority Corrupt (4c1t)}} & \textbf{All Corrupted} \\
\textbf{Dataset} & \textbf{(0c5t)} & MAJ & RRMaj & $\Delta$ & MAJ & RRMaj & $\Delta$ & MAJ & RRMaj & $\Delta$ & MAJ & RRMaj & $\Delta$ & \textbf{(5c0t)} \\ \midrule

\multicolumn{15}{l}{\textbf{Model: Llama-4-Scout-17B}} \\ \midrule
AQuA       & 87.5 & 87.1 & 84.7 & -2.4 & 84.0 & 79.8 & -4.2 & 24.5 & \textbf{63.8} & \textbf{+39.3} & 14.1 & \textbf{36.2} & \textbf{+22.1} & 11.7 \\
GSM8K      & 100.0& 99.4 & \textbf{100.0}& \textbf{+0.6} & 98.2 & 94.5 & -3.7 & 4.3  & \textbf{74.2} & \textbf{+69.9} & 1.8  & \textbf{30.1} & \textbf{+28.3} & 1.6  \\
MATH500    & 91.4 & 92.6 & 92.0 & -0.6 & 90.8 & 84.0 & -6.8 & 18.4 & \textbf{66.3} & \textbf{+47.9} & 6.1  & \textbf{33.7} & \textbf{+27.6} & 9.4  \\
Track7     & 91.4 & 93.3 & 91.4 & -1.9 & 92.0 & 88.3 & -3.7 & 73.6 & 73.6 & 0.0  & 64.4 & \textbf{67.5} & \textbf{+3.1}  & 60.9 \\
Logic7     & 81.6 & 84.7 & \textbf{85.9} & \textbf{+1.2} & 75.5 & 64.4 & -11.0& 6.7  & \textbf{38.7} & \textbf{+31.9} & 3.1  & \textbf{13.5} & \textbf{+10.4} & 3.1  \\
Logic3     & 100.0& 99.4 & \textbf{100.0}& \textbf{+0.6} & 100.0& 94.5 & -5.5 & 23.3 & \textbf{82.2} & \textbf{+58.9} & 4.9  & \textbf{27.0} & \textbf{+22.1} & 3.9  \\ \midrule

\multicolumn{15}{l}{\textbf{Model: Llama-3.3-70B}} \\ \midrule
AQuA       & 83.6 & 81.6 & 81.0 & -0.6 & 77.9 & 74.2 & -3.7 & 1.2  & \textbf{44.2} & \textbf{+43.0} & 0.0  & \textbf{14.1} & \textbf{+14.1} & 0.0  \\
GSM8K      & 98.4 & 96.9 & \textbf{96.9} & \textbf{0.0}  & 96.9 & 94.5 & -2.4 & 1.2  & \textbf{69.3} & \textbf{+68.1} & 0.0  & \textbf{12.3} & \textbf{+12.3} & 0.0  \\
MATH500    & 87.5 & 83.4 & 78.5 & -4.9 & 74.8 & 63.8 & -11.0& 0.6  & \textbf{41.1} & \textbf{+40.5} & 0.0  & \textbf{11.0} & \textbf{+11.0} & 0.0  \\
Track7     & 98.4 & 99.4 & 96.3 & -3.1 & 94.5 & \textbf{96.3} & \textbf{+1.8} & 1.2  & \textbf{79.1} & \textbf{+77.9} & 0.0  & \textbf{25.2} & \textbf{+25.2} & 0.0  \\
Logic7     & 81.6 & 86.5 & \textbf{89.0} & \textbf{+2.5} & 78.5 & 72.4 & -6.1 & 0.0  & \textbf{37.4} & \textbf{+37.4} & 0.0  & \textbf{9.8}  & \textbf{+9.8}  & 0.0  \\
Logic3     & 100.0& 100.0& 98.8 & -1.2 & 100.0& 83.4 & -16.6& 0.6  & \textbf{46.6} & \textbf{+46.0} & 0.0  & \textbf{8.6}  & \textbf{+8.6}  & 0.0  \\ \midrule

\multicolumn{15}{l}{\textbf{Model: Mistral-Small-3.1-24B}} \\ \midrule
AQuA       & 83.6 & 84.0 & 81.6 & -2.4 & 74.2 & 68.7 & -5.5 & 1.8  & \textbf{38.0} & \textbf{+36.2} & 0.0  & \textbf{14.7} & \textbf{+14.7} & 0.0  \\
GSM8K      & 96.1 & 99.4 & 98.2 & -1.2 & 95.1 & 89.0 & -6.1 & 2.5  & \textbf{52.1} & \textbf{+49.6} & 0.0  & \textbf{8.6}  & \textbf{+8.6}  & 0.0  \\
MATH500    & 81.2 & 82.8 & 75.5 & -7.3 & 67.5 & 52.1 & -15.4& 1.8  & \textbf{36.8} & \textbf{+35.0} & 0.0  & \textbf{11.0} & \textbf{+11.0} & 0.0  \\
Track7     & 89.8 & 93.9 & 86.5 & -7.4 & 84.7 & 65.0 & -19.7& 3.1  & \textbf{28.8} & \textbf{+25.7} & 0.6  & \textbf{6.7}  & \textbf{+6.1}  & 0.0  \\
Logic7     & 67.5 & 71.8 & 58.9 & -12.9& 62.0 & 43.6 & -18.4& 0.6  & \textbf{16.6} & \textbf{+16.0} & 0.0  & \textbf{1.8}  & \textbf{+1.8}  & 0.0  \\
Logic3     & 100.0& 98.2 & 94.5 & -3.7 & 95.1 & 67.5 & -27.6& 1.2  & \textbf{22.7} & \textbf{+21.5} & 0.0  & \textbf{1.2}  & \textbf{+1.2}  & 0.0  \\ \midrule

\multicolumn{15}{l}{\textbf{Model: Qwen2.5-Coder-32B}} \\ \midrule
AQuA       & 86.7 & 85.2 & 79.7 & -5.5 & 78.9 & 52.3 & -26.6 & 2.3 & \textbf{28.9} & \textbf{+26.6} & 0.0 & \textbf{3.1} & \textbf{+3.1} & 0.8 \\
GSM8K      & 96.9 & 98.4 & 97.7 & -0.8 & 95.3 & 76.6 & -18.8 & 1.6 & \textbf{22.7} & \textbf{+21.1} & 0.0 & \textbf{0.8} & \textbf{+0.8} & 0.0 \\
MATH500    & 86.7 & 88.3 & 80.5 & -7.8 & 84.4 & 55.5 & -28.9 & 0.0 & \textbf{21.9} & \textbf{+21.9} & 0.0 & \textbf{0.0} & \textbf{+0.0} & 0.0 \\
Track-7    & 75.0 & 72.7 & 71.1 & -1.6 & 61.7 & 53.1 & -8.6  & 1.6 & \textbf{23.4} & \textbf{+21.9} & 0.0 & \textbf{2.3} & \textbf{+2.3} & 0.0 \\
Logic-7    & 77.3 & 83.6 & 76.6 & -7.0 & 68.0 & 46.1 & -21.9& 3.1  & \textbf{19.5} & \textbf{+16.4} & 0.0  & \textbf{0.8}  & \textbf{+0.8}  & 1.6  \\
Logic-3    & 99.2 & 98.4 & 96.9 & -1.6 & 99.2 & 60.9 & -38.3 & 0.8 & \textbf{21.9} & \textbf{+21.1} & 0.0 & \textbf{1.6} & \textbf{+1.6} & 0.0 \\ \midrule

\multicolumn{15}{l}{\textbf{Model: Qwen3-30B}} \\ \midrule
AQuA       & 85.9 & 83.6 & \textbf{85.2} & \textbf{+1.6} & 85.2 & 82.8 & -2.4 & 18.8 & \textbf{23.4} & \textbf{+4.6} & 11.7 & \textbf{14.8} & \textbf{+3.1} & 13.3 \\
GSM8K      & 98.4 & 98.4 & \textbf{99.2} & \textbf{+0.8} & 98.4 & 97.7 & -0.7 & 28.9 & 28.1 & -0.8 & 17.2 & \textbf{21.1} & \textbf{+3.9} & 14.1 \\
MATH500    & 86.7 & 95.3 & \textbf{96.1} & \textbf{+0.8} & 93.8 & 91.4 & -2.4 & 42.2 & 40.6 & -1.6 & 21.9 & \textbf{28.1} & \textbf{+6.2} & 18.0 \\
Track7     & 90.6 & 92.2 & 91.4 & -0.8 & 90.6 & \textbf{91.4} & \textbf{+0.8} & 91.4 & 89.1 & -2.3 & 87.5 & 86.7 & -0.8 & 75.8 \\
Logic7     & 91.4 & 93.0 & \textbf{93.8} & \textbf{+0.8} & 88.3 & 77.3 & -10.9& 16.4 & \textbf{25.0} & \textbf{+8.6}  & 4.7  & \textbf{5.5}  & \textbf{+0.8}  & 4.7  \\
Logic3     & 99.2 & 98.4 & \textbf{99.2} & \textbf{+0.8} & 97.7 & \textbf{97.7} & \textbf{0.0}  & 17.2 & 14.1 & -3.1 & 1.6  & \textbf{3.1}  & \textbf{+1.5} & 9.4  \\ \bottomrule
\end{tabular}\end{adjustbox}
}
\caption{Accuracy under strong imperative injection across six reasoning benchmarks and five model families. We report Majority Voting (MAJ), RRMaj, and the signed accuracy-point difference $\Delta=\mathrm{RRMaj}-\mathrm{MAJ}$. The \texttt{0c5t} and \texttt{5c0t} columns are homogeneous all-honest and all-corrupted reference endpoints. Values are \textbf{bolded} where RRMaj outperforms or has the same score as MAJ.}
\label{tab:main_results_must_bias}
\end{table*}

\section{Comparison with Multi-Agent Debate}
\label{app:debate}

\subsection{Protocols and Compute Regime}

Following the general multi-agent debate framework of \citet{du2024improving}, we implement a two-round, five-agent Debate baseline under the same moderate advisory-injection setting as Table~\ref{tab:moderate_bias_summary}. The experiment uses the same five models, six benchmarks, four corruption configurations, and at least 128 examples per model--benchmark--configuration cell. In Round 1, each agent independently generates a complete response. In Round 2, each agent observes all five Round-1 responses and generates a revised complete response. Each agent retains its original corruption status across both rounds, and the final answer is selected by majority voting over the five Round-2 answers. We omit an external judge to isolate full-response deliberation from judge quality.

We additionally evaluate \textbf{Debate+RR}, a hybrid protocol that keeps Round 1 unchanged but replaces the five independent second-round revisions with five RR-interleaved trajectories. Each second-round trajectory is conditioned on the query and all Round-1 responses, and the same five agents cyclically continue a shared response in $K=100$ token chunks. We then apply majority voting over the five completed second-round RR trajectories. Debate+RR therefore tests whether the benefit of RR persists when agents are first exposed to a conventional debate context.

\ReviewSystemsBegin{}These methods operate in different compute and latency regimes. Under a common cap of $L$ generated tokens per output, MAJ and RRMaj each decode at most $5L$ tokens, whereas two-round Debate and Debate+RR decode at most $10L$ tokens, in addition to the larger Round-2 context prefill. The debate comparison is therefore a higher-compute full-response baseline rather than a strictly compute-matched comparison.\ReviewSystemsEnd{}

\subsection{Aggregate Results}

Across all 120 model--benchmark--configuration cells, Debate+RR improves over Debate in 76 cells, ties in 16, and loses in 28, increasing mean accuracy from $64.0\%$ to $71.2\%$ ($+7.2$ points). The gains are concentrated under majority corruption: Debate+RR improves over Debate by $14.8$ points on average with 53/4/3 wins/ties/losses, whereas its minority-corruption effect is near-neutral at $-0.3$ points with 23/12/25 wins/ties/losses. The three majority-corruption losses are all no larger than $1.56$ points.

Table~\ref{tab:debate_pairwise} compares the four method pairs by corruption regime. Two-round Debate is a higher-compute full-response baseline than MAJ and improves over it by $4.0$ points under majority corruption. RRMaj improves over MAJ by $25.1$ points and exceeds Debate by $21.1$ points in the same regime.
Table~\ref{tab:debate_phase_pattern} further shows the same regime-dependent pattern in two controlled comparisons: RRMaj versus MAJ isolates the effect of replacing independent response generation with RR trajectories, while Debate+RR versus Debate isolates the effect of replacing Debate's independent second-round revisions with RR trajectories. In both cases, the effect is small and mixed under minority corruption but large and systematic once corrupted agents form a majority.

\begin{table*}[t]
\centering
\begin{adjustbox}{max width=\linewidth}\begin{tabular}{llccc}
\toprule
\textbf{Regime} & \textbf{Comparison} & \textbf{Mean $\Delta$} & \textbf{Median $\Delta$} & \textbf{W/T/L} \\
\midrule
\multirow{4}{*}{Minority corruption (\texttt{1c4t}/\texttt{2c3t})}
& RRMaj $-$ MAJ & +0.1 & 0.0 & 26 / 6 / 28 \\
& Debate $-$ MAJ & -5.2 & -1.4 & 25 / 2 / 33 \\
& RRMaj $-$ Debate & +5.3 & +1.8 & 36 / 1 / 23 \\
& Debate+RR $-$ Debate & -0.3 & 0.0 & 23 / 12 / 25 \\
\midrule
\multirow{4}{*}{Majority corruption (\texttt{3c2t}/\texttt{4c1t})}
& RRMaj $-$ MAJ & +25.1 & +21.0 & 55 / 2 / 3 \\
& Debate $-$ MAJ & +4.0 & +4.7 & 41 / 0 / 19 \\
& RRMaj $-$ Debate & +21.1 & +17.7 & 45 / 0 / 15 \\
& Debate+RR $-$ Debate & +14.8 & +9.4 & 53 / 4 / 3 \\
\bottomrule
\end{tabular}\end{adjustbox}
\caption{Pairwise comparisons under moderate advisory injection. Mean and median $\Delta$ are accuracy differences in percentage points, defined as the first method minus the second method. Each corruption regime contains 60 model--benchmark--configuration cells. W/T/L denotes wins/ties/losses of the first method against the second (higher/equal/lower accuracy), counted over those 60 cells using unrounded accuracies.}
\label{tab:debate_pairwise}
\end{table*}

\begin{table*}[t]
\centering
\caption{Configuration-level phase-transition pattern. Each row aggregates 30 model--benchmark cells. W/T/L denotes wins/ties/losses of the first method against the second (higher/equal/lower accuracy), counted over those 30 cells using unrounded accuracies.}
\label{tab:debate_phase_pattern}
\begin{adjustbox}{max width=\linewidth}\begin{tabular}{lcccc}
\toprule
\multirow{2}{*}{\textbf{Configuration}} & \multicolumn{2}{c}{\textbf{RRMaj $-$ MAJ}} & \multicolumn{2}{c}{\textbf{Debate+RR $-$ Debate}} \\
\cmidrule(lr){2-3} \cmidrule(lr){4-5}
& \textbf{Mean $\Delta$} & \textbf{W/T/L} & \textbf{Mean $\Delta$} & \textbf{W/T/L} \\
\midrule
\texttt{1c4t} & -0.04 & 11 / 3 / 16 & +0.73 & 11 / 8 / 11 \\
\texttt{2c3t} & +0.24 & 15 / 3 / 12 & -1.38 & 12 / 4 / 14 \\
\texttt{3c2t} & +28.00 & 27 / 2 / 1 & +17.16 & 28 / 2 / 0 \\
\texttt{4c1t} & +22.20 & 28 / 0 / 2 & +12.40 & 25 / 2 / 3 \\
\bottomrule
\end{tabular}\end{adjustbox}

\end{table*}

\subsection{Full Debate and Debate+RR Results}

Table~\ref{tab:debate_rr_comparison} reports all Debate and Debate+RR accuracies. Over the matched 120 cells, mean accuracy is $77.2\%$ for RRMaj, $71.2\%$ for Debate+RR, and $64.0\%$ for Debate. The phase pattern mirrors the mechanistic result in Sec.~\ref{sec:mechanistic_results}: robustness improves when chunk-level intervention is inserted into the generation process. RRMaj intervenes before complete peer rationales enter the context, while Debate+RR applies RR after Round 1. The hybrid nevertheless substantially improves Debate under majority corruption. This establishes complementarity with the evaluated two-round protocol and motivates broader study across debate variants.

\begin{table*}[t]
\centering
\begin{adjustbox}{max width=\linewidth}\begin{tabular}{l|ccc|ccc|ccc|ccc}
\toprule
& \multicolumn{3}{c|}{\textbf{Minority Corrupt (\texttt{1c4t})}} & \multicolumn{3}{c|}{\textbf{Minority Corrupt (\texttt{2c3t})}} & \multicolumn{3}{c|}{\textbf{Majority Corrupt (\texttt{3c2t})}} & \multicolumn{3}{c}{\textbf{Majority Corrupt (\texttt{4c1t})}} \\
\textbf{Dataset} & Debate & Debate+RR & $\Delta$ & Debate & Debate+RR & $\Delta$ & Debate & Debate+RR & $\Delta$ & Debate & Debate+RR & $\Delta$ \\
\midrule
\multicolumn{13}{l}{\textbf{Model: Llama-3.3-70B}} \\
\midrule
Logic3 & 100.0 & 99.2 & -0.8 & 100.0 & 96.9 & -3.1 & 31.2 & \textbf{89.1} & \textbf{+57.8} & 24.2 & \textbf{74.2} & \textbf{+50.0} \\
Logic7 & 89.8 & 89.8 & +0.0 & 78.9 & \textbf{79.7} & \textbf{+0.8} & 6.2 & \textbf{62.5} & \textbf{+56.2} & 5.5 & \textbf{33.6} & \textbf{+28.1} \\
Track7 & 96.1 & \textbf{99.2} & \textbf{+3.1} & 82.8 & \textbf{92.2} & \textbf{+9.4} & 3.9 & \textbf{65.6} & \textbf{+61.7} & 0.8 & \textbf{16.4} & \textbf{+15.6} \\
AQuA & 76.6 & \textbf{82.8} & \textbf{+6.2} & 71.1 & 69.5 & -1.6 & 10.9 & \textbf{56.2} & \textbf{+45.3} & 10.9 & \textbf{30.5} & \textbf{+19.5} \\
GSM8K & 96.1 & 96.1 & +0.0 & 90.6 & \textbf{93.0} & \textbf{+2.3} & 20.3 & \textbf{73.4} & \textbf{+53.1} & 20.3 & \textbf{65.6} & \textbf{+45.3} \\
MATH500 & 87.5 & 86.7 & -0.8 & 74.2 & \textbf{83.6} & \textbf{+9.4} & 24.2 & \textbf{59.4} & \textbf{+35.2} & 19.5 & \textbf{41.4} & \textbf{+21.9} \\
\textbf{AVG} & 91.0 & \textbf{92.3} & \textbf{+1.3} & 82.9 & \textbf{85.8} & \textbf{+2.9} & 16.1 & \textbf{67.7} & \textbf{+51.6} & 13.5 & \textbf{43.6} & \textbf{+30.1} \\
\midrule
\multicolumn{13}{l}{\textbf{Model: Llama-4-Scout-17B}} \\
\midrule
Logic3 & 99.2 & \textbf{100.0} & \textbf{+0.8} & 99.2 & 99.2 & +0.0 & 93.8 & \textbf{99.2} & \textbf{+5.5} & 84.4 & \textbf{98.4} & \textbf{+14.1} \\
Logic7 & 85.9 & \textbf{92.2} & \textbf{+6.2} & 85.9 & 85.9 & +0.0 & 35.2 & \textbf{73.4} & \textbf{+38.3} & 21.9 & \textbf{49.2} & \textbf{+27.3} \\
Track7 & 82.8 & \textbf{85.2} & \textbf{+2.3} & 81.2 & \textbf{85.2} & \textbf{+3.9} & 68.0 & \textbf{87.5} & \textbf{+19.5} & 58.6 & \textbf{83.6} & \textbf{+25.0} \\
AQuA & 86.7 & 85.9 & -0.8 & 89.8 & 88.3 & -1.6 & 71.1 & \textbf{82.8} & \textbf{+11.7} & 67.2 & \textbf{83.6} & \textbf{+16.4} \\
GSM8K & 98.4 & 98.4 & +0.0 & 98.4 & 97.7 & -0.8 & 89.1 & \textbf{98.4} & \textbf{+9.4} & 76.6 & \textbf{94.5} & \textbf{+18.0} \\
MATH500 & 92.2 & 92.2 & +0.0 & 89.8 & \textbf{90.6} & \textbf{+0.8} & 76.6 & \textbf{85.9} & \textbf{+9.4} & 73.4 & \textbf{82.0} & \textbf{+8.6} \\
\textbf{AVG} & 90.9 & \textbf{92.3} & \textbf{+1.4} & 90.8 & \textbf{91.1} & \textbf{+0.4} & 72.3 & \textbf{87.9} & \textbf{+15.6} & 63.7 & \textbf{81.9} & \textbf{+18.2} \\
\midrule
\multicolumn{13}{l}{\textbf{Model: Qwen2.5-Coder-32B}} \\
\midrule
Logic3 & 88.3 & 87.5 & -0.8 & 61.7 & 53.9 & -7.8 & 10.2 & \textbf{25.0} & \textbf{+14.8} & 2.3 & \textbf{14.8} & \textbf{+12.5} \\
Logic7 & 75.0 & \textbf{83.6} & \textbf{+8.6} & 65.6 & 63.3 & -2.3 & 10.2 & \textbf{33.6} & \textbf{+23.4} & 3.1 & \textbf{18.8} & \textbf{+15.6} \\
Track7 & 68.8 & \textbf{78.1} & \textbf{+9.4} & 34.4 & \textbf{44.5} & \textbf{+10.2} & 5.5 & \textbf{15.6} & \textbf{+10.2} & 4.7 & \textbf{5.5} & \textbf{+0.8} \\
AQuA & 64.1 & 61.7 & -2.3 & 38.3 & 25.8 & -12.5 & 1.6 & \textbf{3.9} & \textbf{+2.3} & 0.8 & \textbf{3.9} & \textbf{+3.1} \\
GSM8K & 76.6 & 71.9 & -4.7 & 33.6 & 25.0 & -8.6 & 4.7 & \textbf{7.0} & \textbf{+2.3} & 0.8 & 0.0 & -0.8 \\
MATH500 & 66.4 & 62.5 & -3.9 & 38.3 & 26.6 & -11.7 & 3.9 & \textbf{5.5} & \textbf{+1.6} & 3.1 & 2.3 & -0.8 \\
\textbf{AVG} & 73.2 & \textbf{74.2} & \textbf{+1.0} & 45.3 & 39.8 & -5.5 & 6.0 & \textbf{15.1} & \textbf{+9.1} & 2.5 & \textbf{7.6} & \textbf{+5.1} \\
\midrule
\multicolumn{13}{l}{\textbf{Model: Qwen3-30B}} \\
\midrule
Logic3 & 100.0 & 100.0 & +0.0 & 100.0 & 100.0 & +0.0 & 97.7 & \textbf{100.0} & \textbf{+2.3} & 98.4 & \textbf{99.2} & \textbf{+0.8} \\
Logic7 & 94.5 & \textbf{98.4} & \textbf{+3.9} & 91.4 & \textbf{92.2} & \textbf{+0.8} & 87.5 & \textbf{89.8} & \textbf{+2.3} & 82.8 & \textbf{84.4} & \textbf{+1.6} \\
Track7 & 100.0 & 100.0 & +0.0 & 99.2 & \textbf{100.0} & \textbf{+0.8} & 100.0 & 100.0 & +0.0 & 100.0 & 100.0 & +0.0 \\
AQuA & 86.7 & \textbf{89.1} & \textbf{+2.3} & 83.6 & \textbf{84.4} & \textbf{+0.8} & 82.8 & 82.8 & +0.0 & 82.0 & 80.5 & -1.6 \\
GSM8K & 100.0 & 100.0 & +0.0 & 100.0 & 100.0 & +0.0 & 97.7 & \textbf{98.4} & \textbf{+0.8} & 99.2 & 99.2 & +0.0 \\
MATH500 & 85.9 & \textbf{96.9} & \textbf{+10.9} & 85.2 & \textbf{94.5} & \textbf{+9.4} & 80.5 & \textbf{89.1} & \textbf{+8.6} & 82.8 & \textbf{87.5} & \textbf{+4.7} \\
\textbf{AVG} & 94.5 & \textbf{97.4} & \textbf{+2.9} & 93.2 & \textbf{95.2} & \textbf{+2.0} & 91.0 & \textbf{93.4} & \textbf{+2.3} & 90.9 & \textbf{91.8} & \textbf{+0.9} \\
\midrule
\multicolumn{13}{l}{\textbf{Model: Mistral-Small-3.1-24B}} \\
\midrule
Logic3 & 98.4 & \textbf{99.2} & \textbf{+0.8} & 96.9 & 96.1 & -0.8 & 74.2 & \textbf{83.6} & \textbf{+9.4} & 61.7 & \textbf{66.4} & \textbf{+4.7} \\
Logic7 & 86.7 & 84.4 & -2.3 & 74.2 & \textbf{78.9} & \textbf{+4.7} & 57.0 & \textbf{60.9} & \textbf{+3.9} & 33.6 & \textbf{49.2} & \textbf{+15.6} \\
Track7 & 94.5 & 94.5 & +0.0 & 92.2 & 85.9 & -6.2 & 71.1 & \textbf{77.3} & \textbf{+6.2} & 62.5 & \textbf{69.5} & \textbf{+7.0} \\
AQuA & 76.6 & 69.5 & -7.0 & 68.0 & 51.6 & -16.4 & 35.2 & \textbf{39.1} & \textbf{+3.9} & 20.3 & \textbf{26.6} & \textbf{+6.2} \\
GSM8K & 97.7 & 93.0 & -4.7 & 88.3 & 74.2 & -14.1 & 39.8 & \textbf{52.3} & \textbf{+12.5} & 24.2 & \textbf{33.6} & \textbf{+9.4} \\
MATH500 & 83.6 & 78.9 & -4.7 & 64.8 & 57.8 & -7.0 & 43.8 & \textbf{50.8} & \textbf{+7.0} & 26.6 & \textbf{29.7} & \textbf{+3.1} \\
\textbf{AVG} & 89.6 & 86.6 & -3.0 & 80.7 & 74.1 & -6.6 & 53.5 & \textbf{60.7} & \textbf{+7.2} & 38.2 & \textbf{45.8} & \textbf{+7.7} \\
\midrule
\multicolumn{13}{l}{\textbf{Overall average across all five models and six benchmarks}} \\
\textbf{AVG} & 87.8 & \textbf{88.6} & \textbf{+0.7} & 78.6 & 77.2 & -1.4 & 47.8 & \textbf{64.9} & \textbf{+17.2} & 41.7 & \textbf{54.1} & \textbf{+12.4} \\
\bottomrule
\end{tabular}\end{adjustbox}
\caption{Comparison with two-round debate under the same moderate advisory-injection setting as Table~\ref{tab:moderate_bias_summary}. Debate+RR uses the same first round as Debate but replaces the second-round revisions with five cyclic chunk-interleaved RR trajectories conditioned on all first-round responses. The final answer is obtained by majority vote over the five second-round outputs. Each Debate and Debate+RR cell contains the same intended seed-42 subset of 128 unique examples. The two methods are evaluated on identical question sets. $\Delta$ denotes Debate+RR $-$ Debate. AVG gives equal weight to the six benchmarks within each model and to all 30 model--benchmark cells overall, using unrounded accuracies before rounding the final result. Bold entries indicate a strict improvement over Debate.}
\label{tab:debate_rr_comparison}
\end{table*}

\clearpage
\raggedbottom
\section{Detailed Ablation Studies and Additional Results}
\label{app:ablation_details}

This section provides the complete protocols, results, and statistical analyses for the four ablations summarized in Sec.~\ref{sec:ablations}: test-time scaling, final-speaker identity, heterogeneous ensembles, and token chunk size.

\subsection{Test-Time Scaling}
\label{app:compute_scaling}

We evaluate Llama-4-Scout-17B on 128 questions per benchmark across six benchmarks, varying the voting budget from $M=1$ to $M=40$. For each question, we estimate voting accuracy by repeatedly sampling from a fixed pool of previously generated outputs. The scaling estimates use 50 resampling trials per cell. The results are shown in Tables~\ref{tab:app_ablation_scaling} and~\ref{tab:app_ablation_per_dataset}.

For MAJ, $M=1$ denotes the expected correctness of one randomly selected agent response. For a configuration with $c$ corrupted and $t$ honest agents,
\begin{align*}
p_{\mathrm{MAJ},1}
=&\frac{c}{c+t}P(\text{corrupted response is correct})\\
&+\frac{t}{c+t}P(\text{honest response is correct}).
\end{align*}
For RRMaj, $M=1$ is the correctness probability of one RR trajectory. For the strong \texttt{3c2t} results, we compute both single-output expectations directly from the output pools. MAJ draws 24 corrupted and 16 honest responses within their respective role pools. RRMaj draws 40 trajectories. MAJ breaks ties uniformly at random, while RRMaj selects the first tied answer encountered in the sampled sequence. This estimates a 40-vote decision from the empirical pool, using the same saved generations at both endpoints. We average the six benchmark accuracies equally.

Under strong \texttt{3c2t} corruption, average MAJ accuracy is lower at $M=40$ than at $M=1$, whereas average RRMaj accuracy is higher. Under strong \texttt{4c1t}, the precision of RRMaj decreases as $M$ increases in the scaling analysis. These results show that the benefit of additional voting depends on the trajectory distribution and corruption configuration.

\begin{table*}[tbh]
\centering
\begin{adjustbox}{max width=\linewidth}\begin{tabular}{ll|cccccc|c}
\toprule
\textbf{Bias} & \textbf{Method} & \textbf{M=1} & \textbf{M=5} & \textbf{M=10} & \textbf{M=20} & \textbf{M=30} & \textbf{M=40} & \textbf{$\Delta$(40-1)} \\
\midrule
\multicolumn{9}{c}{\textbf{Critical Corruption (3c2t): The Phase Transition}} \\
\midrule
\multirow{3}{*}{\textbf{STRONG}}
& MAJ & $45.85$ & $22.2_{\pm1.0}$ & $23.7_{\pm0.8}$ & $25.3_{\pm0.7}$ & $25.8_{\pm0.6}$ & $27.00_{\pm0.64}$ & -18.85 \\
& RRMaj & $59.27$ & $67.7_{\pm1.2}$ & $70.7_{\pm1.0}$ & $73.6_{\pm1.0}$ & $74.7_{\pm0.7}$ & $72.80_{\pm0.81}$ & \textbf{+13.54} \\
& \textit{$\Delta$(RRMaj-MAJ)} & \textit{+13.42} & \textit{+45.5} & \textit{+47.0} & \textit{+48.3} & \textit{+48.9} & \textit{+45.81} & - \\
\midrule
\multirow{3}{*}{\textbf{MODERATE}}
& MAJ & $77.8_{\pm1.9}$ & $78.0_{\pm0.7}$ & $78.2_{\pm0.6}$ & $79.3_{\pm0.5}$ & $79.5_{\pm0.4}$ & $79.7_{\pm0.4}$ & +1.9 \\
& RRMaj & $81.8_{\pm1.0}$ & $88.1_{\pm0.6}$ & $89.6_{\pm0.6}$ & $90.4_{\pm0.5}$ & $90.6_{\pm0.4}$ & $90.6_{\pm0.1}$ & \textbf{+8.8} \\
& \textit{$\Delta$(RRMaj-MAJ)} & \textit{+4.0} & \textit{+10.1} & \textit{+11.4} & \textit{+11.1} & \textit{+11.1} & \textit{+10.9} & - \\
\midrule
\multicolumn{9}{c}{\textbf{Extreme Corruption (4c1t): The Condorcet Trap}} \\
\midrule
\multirow{3}{*}{\textbf{STRONG}}
& MAJ & $29.9_{\pm2.1}$ & $17.0_{\pm0.6}$ & $16.9_{\pm0.6}$ & $17.0_{\pm0.4}$ & $17.1_{\pm0.4}$ & $17.0_{\pm0.3}$ & -12.9 \\
& RRMaj & $40.3_{\pm1.4}$ & $39.0_{\pm1.3}$ & $37.0_{\pm1.2}$ & $35.2_{\pm0.9}$ & $34.4_{\pm0.6}$ & $33.5_{\pm0.3}$ & -6.8 \\
& \textit{$\Delta$(RRMaj-MAJ)} & \textit{+10.4} & \textit{+22.0} & \textit{+20.1} & \textit{+18.2} & \textit{+17.3} & \textit{+16.5} & - \\
\midrule
\multirow{3}{*}{\textbf{MODERATE}}
& MAJ & $72.4_{\pm1.5}$ & $73.2_{\pm0.7}$ & $73.4_{\pm0.6}$ & $74.2_{\pm0.4}$ & $74.4_{\pm0.4}$ & $74.4_{\pm0.4}$ & +2.0 \\
& RRMaj & $75.3_{\pm1.2}$ & $80.5_{\pm1.0}$ & $81.9_{\pm0.6}$ & $82.5_{\pm0.6}$ & $83.0_{\pm0.6}$ & $82.7_{\pm0.2}$ & \textbf{+7.4} \\
& \textit{$\Delta$(RRMaj-MAJ)} & \textit{+2.9} & \textit{+7.3} & \textit{+8.5} & \textit{+8.4} & \textit{+8.6} & \textit{+8.3} & - \\
\bottomrule
\end{tabular}\end{adjustbox}
\caption{Llama-4-Scout-17B accuracy (\%), averaged equally across six benchmarks. }
\label{tab:app_ablation_scaling}
\end{table*}

\begin{table}[!tbp]
\centering
\begin{adjustbox}{max width=\linewidth}\begin{tabular}{l|rr|rr}
\toprule
\multirow{2}{*}{\textbf{Dataset}} & \multicolumn{2}{c|}{\textbf{MAJ}} & \multicolumn{2}{c}{\textbf{RRMaj}} \\
& \textbf{M=1} & \textbf{M=40} & \textbf{M=1} & \textbf{M=40} \\
\midrule
AQuA & 42.88 & 24.65 & 56.53 & \textbf{69.93} \\
GSM8K & 41.56 & 8.35 & 65.75 & \textbf{84.58} \\
Logic7 & 34.78 & 6.05 & 43.06 & \textbf{47.06} \\
Logic3 & 41.66 & 25.46 & 61.63 & \textbf{79.12} \\
MATH500 & 41.16 & 15.12 & 56.66 & \textbf{68.81} \\
Track7 & 73.06 & 82.34 & 71.97 & \textbf{87.32} \\
\midrule
\textbf{Average} & 45.85 & 27.00 & 59.27 & \textbf{72.80} \\
\bottomrule
\end{tabular}\end{adjustbox}
\caption{Per-benchmark Llama-4-Scout-17B accuracy under strong \texttt{3c2t} corruption. $M=1$ is the single-output expectation from 25 saved outputs. $M=40$ averages 1,000 resampling trials with replacement, each aggregating 40 outputs.}
\label{tab:app_ablation_per_dataset}
\end{table}

Under strong \texttt{3c2t} corruption, RRMaj accuracy rises from $59.27\%$ at $M=1$ to $72.80\%$ at $M=40$, a gain of $13.54$ percentage points computed before rounding. MAJ accuracy falls from $45.85\%$ to $27.00\%$, and the $M=40$ RRMaj advantage is $45.81$ points. Each of the six benchmarks shows higher RRMaj accuracy at $M=40$ than at $M=1$. Under strong \texttt{4c1t}, the scaling analysis shows decreasing RRMaj accuracy as votes are added. Together, these results show why trajectory quality matters when allocating additional votes.

\subsection{Final-Speaker Identity}
\label{app:final_speaker}

We test whether RR's recovery can be explained by an honest agent simply generating the last chunk and producing the final answer in the sequence. For each evaluated configuration, we partition trajectories according to whether the final chunk was generated by an honest or corrupted speaker (Honest-Last and Corrupted-Last) and compare the corresponding correctness rates using two-sided Fisher's exact tests.

Table~\ref{tab:app_last_speaker} reports the conditional accuracies, sample sizes, differences, and $p$-values. All tested comparisons have $p>0.05$. Correct answers occur with both final-speaker roles, and these tests do not detect a difference between their conditional correctness rates.

\begin{table*}[ht]
\centering
\begin{adjustbox}{max width=\linewidth}\begin{tabular}{l|cccc|cccc}
\toprule
 & \multicolumn{4}{c|}{\textbf{Llama-4-Scout-17B}} & \multicolumn{4}{c}{\textbf{Mistral-Small-3.1-24B}} \\
\textbf{Config ($\rho$)} & \textbf{Honest-Last} & \textbf{Corrupted-Last} & \textbf{$\Delta$(\%)} & \textbf{$p$-value} & \textbf{Honest-Last} & \textbf{Corrupted-Last} & \textbf{$\Delta$(\%)} & \textbf{$p$-value} \\ \midrule
1c4t (0.2) & 81.12\% ($n$=1176) & 82.69\% ($n$=260) & -1.57 & 0.5978 & 70.39\% ($n$=1074) & 67.60\% ($n$=250) & +2.79 & 0.4003 \\
2c3t (0.4) & 75.75\% ($n$=870)  & 78.45\% ($n$=566) & -2.70 & 0.2498 & 57.41\% ($n$=817)  & 59.50\% ($n$=479) & -2.09 & 0.4841 \\
3c2t (0.6) & 71.68\% ($n$=565)  & 70.84\% ($n$=871) & +0.84 & 0.7657 & 45.62\% ($n$=548)  & 45.59\% ($n$=748) & +0.03 & 1.0000 \\
4c1t (0.8) & 65.41\% ($n$=266)  & 66.50\% ($n$=1170)& -1.09 & 0.7738 & 34.77\% ($n$=279)  & 33.33\% ($n$=1017)& +1.44 & 0.6681 \\ \bottomrule
\end{tabular}\end{adjustbox}
\caption{Correctness conditional on the role of the speaker generating the final chunk. $\Delta$ is Honest-Last minus Corrupted-Last accuracy. All reported Fisher exact-test $p$-values exceed $0.05$, indicating that no statistically significant association is detected.}
\label{tab:app_last_speaker}
\end{table*}

\subsection{Heterogeneous-Ensemble Results}
\label{app:hetero_full}

We evaluate two classes of heterogeneous ensembles: mixed-capacity Llama-3.3-70B/Llama-3-8B ensembles with $N\in\{3,5,7\}$, and one diverse $N=5$ ensemble containing 70B, 17B, 24B, 12B, and 8B models. This focused two-benchmark ablation uses \texttt{GSM8K} and \texttt{Logic-3} to vary ensemble size, capacity mix, and corruption placement more densely than the main six-benchmark grid.

Each reported accuracy is the equal-weight two-benchmark macro-average
\begin{equation}
\operatorname{Acc}_{\mathrm{macro}}
=
\frac{1}{2}
\left(
\operatorname{Acc}_{\texttt{GSM8K}}
+
\operatorname{Acc}_{\texttt{Logic-3}}
\right).
\end{equation}
Each corruption condition contains 36 questions per benchmark, for 72 question-level outcomes per method. Because the two benchmarks have equal support, the macro-average also equals the pooled accuracy over these 72 outcomes. Differences are computed from the unrounded accuracies and can therefore differ by $0.1$ point from subtracting the displayed one-decimal values. The tables report MAJ and RRMaj accuracy under strong imperative corruption and test whether recovery is determined solely by honest-agent headcount or model capacity. In the row labels, ``corrupted'' and ``honest'' specify the model multiplicities assigned to the two roles. In the diverse $N=5$ table, the bracketed list specifies the corrupted models and all unlisted models are honest. The abbreviated aliases match the implementation logs.

\paragraph{Capacity and corruption placement.} MAJ assigns one equal vote to each completed answer: model capacity can affect answer quality, but it does not affect voting weight. In the $N=5$ ensemble with three corrupted 70B agents and two honest 8B agents, MAJ reaches $18.1\%$, whereas RRMaj reaches $86.1\%$, an approximately $68$-point gain. This result is consistent with a smaller-model honest minority exerting substantial influence when its continuations are inserted into the shared trajectory.

\paragraph{Configuration dependence.} Large honest-model capacity is one ingredient rather than the sole determinant of recovery. With four corrupted 8B agents and one honest 70B agent, accuracy rises from $45.8\%$ under MAJ to $59.7\%$ under RRMaj, a gain of $13.9$ points. In the diverse ensemble, leaving only the 70B agent honest changes accuracy from $87.5\%$ to $81.9\%$ ($-5.6$ points). Together with the larger positive cells, this variation shows that ensemble composition, corruption placement, and repeated intervention opportunities jointly shape recovery.

\paragraph{Aggregate pattern.} Across the 20 majority-corruption conditions in Tables~\ref{tab:app_hetero_n3}--\ref{tab:app_hetero_diverse}, $\Delta=\mathrm{RRMaj}-\mathrm{MAJ}$ is positive in 18 conditions and has an unweighted mean of $+24.9$ percentage points. This broad descriptive pattern is consistent with honest models introducing constraints that later speakers preserve or extend. The configuration-level variation motivates future factorial experiments that isolate capacity from model family and corruption placement.

\begin{table*}[t]
\centering
\begin{adjustbox}{max width=\linewidth}\begin{tabular}{l|ccc}
\toprule
\textbf{Corruption Pattern} & \textbf{MAJ} & \textbf{RRMaj} & \textbf{Difference ($\Delta$)} \\ \midrule
\multicolumn{4}{l}{\textit{\textbf{N3\_2big1small (2$\times$70B, 1$\times$8B)}}} \\ \midrule
0c3t (0$\times$70B+0$\times$8B corrupted | 2$\times$70B+1$\times$8B honest) & 100.0\% & -- & -- \\
1c2t (0$\times$70B+1$\times$8B corrupted | 2$\times$70B+0$\times$8B honest) & 97.2\% & 88.9\% & -8.3 \\
1c2t (1$\times$70B+0$\times$8B corrupted | 1$\times$70B+1$\times$8B honest) & 84.7\% & \textbf{94.4\%} & \textbf{+9.7} \\
2c1t (1$\times$70B+1$\times$8B corrupted | 1$\times$70B+0$\times$8B honest) & 58.3\% & \textbf{75.0\%} & \textbf{+16.7} \\
2c1t (2$\times$70B+0$\times$8B corrupted | 0$\times$70B+1$\times$8B honest) & 19.4\% & \textbf{68.1\%} & \textbf{+48.6} \\ \midrule
\multicolumn{4}{l}{\textit{\textbf{N3\_2small1big (2$\times$8B, 1$\times$70B)}}} \\ \midrule
0c3t (0$\times$70B+0$\times$8B corrupted | 1$\times$70B+2$\times$8B honest) & 90.3\% & -- & -- \\
1c2t (0$\times$70B+1$\times$8B corrupted | 1$\times$70B+1$\times$8B honest) & 86.1\% & 80.6\% & -5.6 \\
1c2t (1$\times$70B+0$\times$8B corrupted | 0$\times$70B+2$\times$8B honest) & 72.2\% & \textbf{80.6\%} & \textbf{+8.3} \\
2c1t (0$\times$70B+2$\times$8B corrupted | 1$\times$70B+0$\times$8B honest) & 68.1\% & \textbf{69.4\%} & \textbf{+1.4} \\
2c1t (1$\times$70B+1$\times$8B corrupted | 0$\times$70B+1$\times$8B honest) & 36.1\% & \textbf{59.7\%} & \textbf{+23.6} \\ \bottomrule
\end{tabular}\end{adjustbox}
\caption{Heterogeneous $N=3$ ensembles under strong imperative corruption. Each entry is the equal-weight macro-average accuracy across \texttt{GSM8K} and \texttt{Logic-3}, with 36 questions per benchmark and 72 question-level outcomes per method and condition. We report $\Delta=\mathrm{RRMaj}-\mathrm{MAJ}$. Boldface marks strictly positive RRMaj differences. Dashes indicate that RRMaj was not evaluated for the all-honest reference condition.}
\label{tab:app_hetero_n3}
\end{table*}

\begin{table*}[t]
\centering
\begin{adjustbox}{max width=\linewidth}\begin{tabular}{l|ccc}
\toprule
\textbf{Corruption Pattern} & \textbf{MAJ} & \textbf{RRMaj} & \textbf{Difference ($\Delta$)} \\ \midrule
\multicolumn{4}{l}{\textit{\textbf{N5\_3big2small (3$\times$70B, 2$\times$8B)}}} \\ \midrule
0c5t (0$\times$70B+0$\times$8B corrupted | 3$\times$70B+2$\times$8B honest) & 97.2\% & -- & -- \\
1c4t (0$\times$70B+1$\times$8B corrupted | 3$\times$70B+1$\times$8B honest) & 100.0\% & 93.1\% & -6.9 \\
1c4t (1$\times$70B+0$\times$8B corrupted | 2$\times$70B+2$\times$8B honest) & 98.6\% & 97.2\% & -1.4 \\
2c3t (0$\times$70B+2$\times$8B corrupted | 3$\times$70B+0$\times$8B honest) & 98.6\% & 83.3\% & -15.3 \\
2c3t (2$\times$70B+0$\times$8B corrupted | 1$\times$70B+2$\times$8B honest) & 73.6\% & \textbf{91.7\%} & \textbf{+18.1} \\
3c2t (1$\times$70B+2$\times$8B corrupted | 2$\times$70B+0$\times$8B honest) & 66.7\% & \textbf{76.4\%} & \textbf{+9.7} \\
3c2t (3$\times$70B+0$\times$8B corrupted | 0$\times$70B+2$\times$8B honest) & 18.1\% & \textbf{86.1\%} & \textbf{+68.1} \\
4c1t (2$\times$70B+2$\times$8B corrupted | 1$\times$70B+0$\times$8B honest) & 31.9\% & \textbf{65.3\%} & \textbf{+33.3} \\
4c1t (3$\times$70B+1$\times$8B corrupted | 0$\times$70B+1$\times$8B honest) & 15.3\% & \textbf{62.5\%} & \textbf{+47.2} \\ \midrule
\multicolumn{4}{l}{\textit{\textbf{N5\_1big4small (1$\times$70B, 4$\times$8B)}}} \\ \midrule
0c5t (0$\times$70B+0$\times$8B corrupted | 1$\times$70B+4$\times$8B honest) & 93.1\% & -- & -- \\
1c4t (0$\times$70B+1$\times$8B corrupted | 1$\times$70B+3$\times$8B honest) & 88.9\% & 86.1\% & -2.8 \\
1c4t (1$\times$70B+0$\times$8B corrupted | 0$\times$70B+4$\times$8B honest) & 76.4\% & \textbf{93.1\%} & \textbf{+16.7} \\
2c3t (0$\times$70B+2$\times$8B corrupted | 1$\times$70B+2$\times$8B honest) & 81.9\% & 81.9\% & +0.0 \\
2c3t (1$\times$70B+1$\times$8B corrupted | 0$\times$70B+3$\times$8B honest) & 76.4\% & \textbf{80.6\%} & \textbf{+4.2} \\
3c2t (0$\times$70B+3$\times$8B corrupted | 1$\times$70B+1$\times$8B honest) & 75.0\% & 72.2\% & -2.8 \\
3c2t (1$\times$70B+2$\times$8B corrupted | 0$\times$70B+2$\times$8B honest) & 59.7\% & \textbf{75.0\%} & \textbf{+15.3} \\
4c1t (0$\times$70B+4$\times$8B corrupted | 1$\times$70B+0$\times$8B honest) & 45.8\% & \textbf{59.7\%} & \textbf{+13.9} \\
4c1t (1$\times$70B+3$\times$8B corrupted | 0$\times$70B+1$\times$8B honest) & 45.8\% & \textbf{50.0\%} & \textbf{+4.2} \\ \bottomrule
\end{tabular}\end{adjustbox}
\caption{Heterogeneous $N=5$ ensembles under strong imperative corruption. Each entry is the equal-weight macro-average accuracy across \texttt{GSM8K} and \texttt{Logic-3}, with 36 questions per benchmark and 72 question-level outcomes per method and condition. We report $\Delta=\mathrm{RRMaj}-\mathrm{MAJ}$. Boldface marks strictly positive RRMaj differences. Dashes indicate that RRMaj was not evaluated for the all-honest reference condition.}
\label{tab:app_hetero_n5}
\end{table*}

\begin{table*}[t]
\centering
\begin{adjustbox}{max width=\linewidth}\begin{tabular}{l|ccc}
\toprule
\textbf{Corruption Pattern} & \textbf{MAJ} & \textbf{RRMaj} & \textbf{Difference ($\Delta$)} \\ \midrule
\multicolumn{4}{l}{\textit{\textbf{N5\_diverse (70B, 17B, 24B, 12B, 8B)}}} \\ \midrule
0c5t (0 corrupted) & 100.0\% & -- & -- \\
1c4t (corrupted=[llama-70b]) & 100.0\% & 100.0\% & +0.0 \\
1c4t (corrupted=[llama-8b]) & 100.0\% & 98.6\% & -1.4 \\
2c3t (corrupted=[gemma-12b, llama-8b]) & 97.2\% & \textbf{100.0\%} & \textbf{+2.8} \\
2c3t (corrupted=[llama-70b, scout-17b]) & 93.1\% & \textbf{95.8\%} & \textbf{+2.8} \\
3c2t (corrupted=[llama-70b, scout-17b, mistral-24b]) & 63.9\% & \textbf{84.7\%} & \textbf{+20.8} \\
3c2t (corrupted=[mistral-24b, gemma-12b, llama-8b]) & 87.5\% & \textbf{88.9\%} & \textbf{+1.4} \\
4c1t (corrupted=[llama-70b, scout-17b, mistral-24b, gemma-12b]) & 55.6\% & \textbf{72.2\%} & \textbf{+16.7} \\
4c1t (corrupted=[scout-17b, mistral-24b, gemma-12b, llama-8b]) & 87.5\% & 81.9\% & -5.6 \\ \midrule
\multicolumn{4}{l}{\textit{\textbf{N7\_4big3small (4$\times$70B, 3$\times$8B)}}} \\ \midrule
0c7t (0$\times$70B+0$\times$8B corrupted | 4$\times$70B+3$\times$8B honest) & 98.6\% & -- & -- \\
1c6t (0$\times$70B+1$\times$8B corrupted | 4$\times$70B+2$\times$8B honest) & 100.0\% & 98.6\% & -1.4 \\
1c6t (1$\times$70B+0$\times$8B corrupted | 3$\times$70B+3$\times$8B honest) & 98.6\% & 97.2\% & -1.4 \\
3c4t (0$\times$70B+3$\times$8B corrupted | 4$\times$70B+0$\times$8B honest) & 100.0\% & 86.1\% & -13.9 \\
3c4t (3$\times$70B+0$\times$8B corrupted | 1$\times$70B+3$\times$8B honest) & 76.4\% & \textbf{95.8\%} & \textbf{+19.4} \\
4c3t (2$\times$70B+2$\times$8B corrupted | 2$\times$70B+1$\times$8B honest) & 68.1\% & \textbf{86.1\%} & \textbf{+18.1} \\
4c3t (4$\times$70B+0$\times$8B corrupted | 0$\times$70B+3$\times$8B honest) & 23.6\% & \textbf{86.1\%} & \textbf{+62.5} \\
5c2t (4$\times$70B+1$\times$8B corrupted | 0$\times$70B+2$\times$8B honest) & 16.7\% & \textbf{83.3\%} & \textbf{+66.7} \\
6c1t (3$\times$70B+3$\times$8B corrupted | 1$\times$70B+0$\times$8B honest) & 22.2\% & \textbf{59.7\%} & \textbf{+37.5} \\ \bottomrule
\end{tabular}\end{adjustbox}
\caption{Diverse $N=5$ and heterogeneous $N=7$ ensembles under strong imperative corruption. Each entry is the equal-weight macro-average accuracy across \texttt{GSM8K} and \texttt{Logic-3}, with 36 questions per benchmark and 72 question-level outcomes per method and condition. We report $\Delta=\mathrm{RRMaj}-\mathrm{MAJ}$. Boldface marks strictly positive RRMaj differences. Dashes indicate that RRMaj was not evaluated for the all-honest reference condition.}
\label{tab:app_hetero_diverse}
\end{table*}

\clearpage
\subsection{Token Chunk Size (\texorpdfstring{$K$}{K})}
\label{app:chunk_size}

Table~\ref{tab:app_k_ablation} reports a chunk-size sweep for Llama-3-8B, Gemma-3-12B, Mistral-Small-3.1-24B, and Llama-3.3-70B. Each model forms a homogeneous three-agent ensemble with two corrupted agents and one honest agent. Each question is evaluated with one RR trajectory ($N=3$, $M=1$), using 100 questions per model and benchmark at each $K\in\{10,30,75,150,300,500\}$. Benchmark entries average accuracy equally across the four models, and AVG averages the three displayed benchmark entries. Average accuracy rises sharply from $9.6\%$ at $K=10$ to $31.8\%$ at $K=75$, reaches $33.8\%$ at both $K=150$ and $K=300$, and declines to $32.3\%$ at $K=500$.

The curve is consistent with a trade-off between reasoning coherence and intervention frequency. Very small chunks can interrupt agents before they complete useful intermediate steps. Larger chunks give each speaker more room to develop a continuation, while reducing the frequency of honest interventions and giving corrupted speakers longer uninterrupted turns. At still larger values, some trajectories may terminate before all $N$ agents contribute, making early speakers more influential. The evaluated range through $K=500$ captures the rise, plateau, and onset of the decline in accuracy.

The main five-agent experiments keep $K=100$ fixed. The separate three-agent sweep characterizes the coherence--intervention trade-off and motivates evaluating intermediate chunk sizes when selecting a setting for a new model or task.

\begin{table}[H]
\centering
\begin{adjustbox}{max width=\linewidth}\begin{tabular}{l|ccc|c}
\toprule
$K$ (tokens) & Logic7 & Logic3 & Track7 & \textbf{AVG} \\ \midrule
10  & 6.8\%  & 10.2\% & 11.8\% & 9.6\%  \\
30  & 10.5\% & 27.5\% & 28.2\% & 22.1\% \\
75  & 18.8\% & 32.2\% & 44.5\% & 31.8\% \\
150 & 19.0\% & 33.5\% & 48.8\% & 33.8\% \\
300 & 21.8\% & 34.8\% & 44.8\% & 33.8\% \\
500 & 21.2\% & 34.2\% & 41.5\% & 32.3\% \\ \bottomrule
\end{tabular}\end{adjustbox}
\caption{Chunk-size ablation across four models, each using one three-agent RR trajectory with two corrupted agents and one honest agent ($N=3$, $M=1$). Each benchmark entry averages four model accuracies, each based on 100 questions. AVG is the arithmetic mean of the three displayed benchmark entries. Accuracy reaches a plateau at $K=150$--$300$ and declines at $K=500$.}
\label{tab:app_k_ablation}
\end{table}

\enlargethispage{\baselineskip}
\section{Disclosure of AI Assistance}
AI assistance is used in language editing, \LaTeX{} formatting, and implementing, running, and checking experiment code under the authors' direction. The authors are responsible for the final manuscript, experiments, and references.

\end{document}